\makeatletter\let\showhyphens\@undefined\makeatother
\documentclass{article}

\usepackage{natbib}

\usepackage{arxiv}

\usepackage[utf8]{inputenc} 
\usepackage[T1]{fontenc}    
\usepackage{hyperref}       
\usepackage{url}            
\usepackage{booktabs}     
\usepackage{amsthm}
\usepackage{amsfonts}       
\usepackage{nicefrac}       
\usepackage{float} 
\usepackage{microtype}      
\usepackage{lipsum}
\usepackage{graphicx}
\newtheorem{theorem}{Theorem}
\newtheorem{lemma}{Lemma}
\newtheorem{corollary}{Corollary}
\newtheorem{remark}{Remark}
\newtheorem{definition}{Definition}
\newtheorem{assumption}{Assumption}
\makeatother
\usepackage{algorithmic}
\usepackage{algorithm}
\usepackage{array}
\usepackage{amsmath}

\usepackage[caption=false,font=normalsize,labelfont=sf,textfont=sf]{subfig}
\usepackage{textcomp}
\usepackage{stfloats}
\usepackage{url}
\usepackage{verbatim}
\usepackage{graphicx}
\usepackage{cleveref}

\usepackage{multirow}

\usepackage{lastpage}
\graphicspath{ {./images/} }

\title{Fisher-Geometric Sharpness and the Implicit Bias of SGD toward Flat Minima}

\author{
 Md Sakir Ahmed \\
  Department of Electronics and Communication Technology\\
  Gauhati University\\
  Guwahati, Assam, India \\
  \texttt{ahmedsakir717@gmail.com} \\
  \And
 Kumaresh Sarmah \\
  Department of Electronics and Communication Technology\\
  Gauhati University\\
  Guwahati, Assam, India \\
  \texttt{kumaresh@gauhati.ac.in} \\
  \And
 Hemen Dutta \\
  Department of Mathematics\\
  Gauhati University\\
  Guwahati, Assam, India \\
  \texttt{duttah@gauhati.ac.in} \\
}

\begin{document}
\maketitle
\begin{abstract}
A widely held intuition in deep learning is that stochastic gradient descent (SGD) implicitly favors flat minima and that flat minima generalize better, but the standard Euclidean measures of flatness like the trace or maximum eigenvalue of the loss Hessian are not invariant under reparametrizations that preserve the network function, which undermines the theoretical foundations of this narrative. In this study we resolve this issue by grounding flatness in the Riemannian geometry of the statistical manifold induced by the Fisher Information Matrix (FIM). We define Riemannian sharpness mathematically and prove that it is invariant under smooth, function-preserving reparametrizations, which directly addresses the critique of Dinh et al. in the paper ``Sharp minima can generalize for deep nets''. We note that this invariance is a property of the true FIM; the diagonal empirical estimator used in practice (and in all experiments below) inherits invariance only approximately, and exact invariance under arbitrary reparametrizations would require structured estimators such as K-FAC. We formalize the gradient noise of mini-batch SGD as having a covariance structure proportional to the FIM, derive the stationary distribution of the resulting stochastic differential equation and then show that the probability mass is exponentially concentrated at Riemannian-flat minima. A PAC-Bayes generalization bound controlled explicitly by $\mathcal{S}_R$ formally links this geometric bias to test performance. Our experiments on MNIST and CIFAR-10 confirm that $\mathcal{S}_R$ reliably tracks generalization in ways that Euclidean sharpness does not, and that its scaling with $\eta/B$ matches the theoretical predictions. Together these results provide a rigorous, reparametrization-invariant account of why flat minima generalize.
\end{abstract}


\section{Introduction}

A central question in deep learning theory is that why stochastic gradient descent (SGD) finds solutions that generalize well despite operating in heavily over-parameterized regimes where infinitely many zero-loss solutions exist. A widely held intuition is that SGD implicitly prefers \emph{flat} minima in which the regions of the loss landscape at which the curvature is low along with flat minima that corresponds to models with better generalization. This flatness bias has been empirically documented \cite{keskar2017large} and partially formalized through continuous-time SDE analyses \cite{li2017stochastic} but the theoretical foundation of this intuition remains fragile. The central difficulty was identified by \cite{dinh2017sharp} that any Euclidean-based flatness measure such as the trace or maximum eigenvalue of the loss Hessian is not invariant under reparametrization. For a two-layer network the rescaling weights by $\alpha$ and $\alpha^{-1}$ preserves the function exactly while making the Euclidean sharpness arbitrarily large or small. Since the Euclidean sharpness is not an intrinsic property of the model, therefore, it cannot be meaningfully linked to generalization. This creates a fundamental gap, i.e., if the flatness measure is ill-defined then the entire flatness-generalization narrative is called into question. Partial steps toward invariant measures have been taken \cite{jang2022reparametrization}, \cite{kristiadi2023geometry}, but a unified account connecting invariant sharpness along with the SGD's implicit bias and generalization bounds has remained lacking.

In this study, we have tried to resolve this gap by replacing the Euclidean metric on parameter space with the \textbf{Fisher Information Matrix (FIM)}, which endows the parameter space with a Riemannian structure that is intrinsic to the model's predictive distribution. We then proceed to define \emph{Riemannian sharpness} as $\mathcal{S}_R(\theta^*) = \mathrm{tr}(G(\theta^*)^{-1}\nabla^2\mathcal{L}(\theta^*))$, where $G(\theta^*)$ is the FIM at the minimum. This quantity measures curvature relative to the natural geometry of the statistical manifold and is invariant under smooth reparametrizations that preserve the network function which directly addresses the critique of \cite{dinh2017sharp}. Our main contribution in this study is a rigorous theoretical framework showing that mini-batch SGD has an \emph{implicit geometric bias} toward minima with small Riemannian sharpness. We have formalized the gradient noise of mini-batch SGD as having a covariance structure proportional to the FIM after which we have attempted to derive the stationary distribution of the resulting SDE and have shown that probability mass is exponentially concentrated at Riemannian-flat minima. We further connect the Riemannian flatness to generalization via a PAC-Bayes bound which provides a formally grounded and reparametrization-invariant account of why flat minima generalize so well.

\subsection*{Contributions}

\begin{itemize}
    \item We have defined Riemannian sharpness on the statistical manifold and showed its invariance under reparametrization (Lemma~1), attempting to resolve the fundamental critique of \cite{dinh2017sharp}.

    \item We are explicit about the gap between the theoretical invariance of $\mathcal{S}_R$ (which holds for the true FIM) and its empirical estimator (which is not exactly invariant; see Section~\ref{sec:reparam}).

    \item We have derived the local stationary distribution of the mini-batch SGD SDE and then proved that it assigns exponentially greater mass to Riemannian-flat minima (Theorem~1, Corollary~1).

    \item We have established a PAC-Bayes generalization bound that is explicitly controlled by Riemannian sharpness (Corollary~2) which formally links flatness to test performance.

    \item We have provided empirical validation on MNIST and CIFAR-10, confirming that Riemannian sharpness tracks generalization across optimizers, batch sizes, and learning rates in ways that Euclidean sharpness does not.
\end{itemize}

Unlike prior work \cite{jang2022reparametrization, kristiadi2023geometry} which establishes reparametrization invariance in isolation, this work is the first to unify invariant sharpness with the SGD stationary distribution and a PAC-Bayes generalization bound in a single framework.

\section{Related Work}

\subsection{Flat Minima and Generalization}

The conjecture that the flat minima generalize better than sharp ones was first formalized by \cite{hochreiter1997flat} who used a minimum description length (MDL) argument to show that low-curvature regions of the loss surface correspond to simpler models with lower expected overfitting. This insight motivated algorithms to be explicitly designed to seek flat minima. More recently \cite{keskar2017large} provided influential empirical evidence that large-batch SGD tends to converge to sharp minima, while small-batch SGD consistently finds flatter ones which establishes a direct link between batch size, loss landscape geometry and generalization. \cite{chaudhari2017entropy} had built on this observation with Entropy-SGD which is an algorithm that constructs a local-entropy-based objective to bias gradient descent toward wide and flat valleys of the energy landscape with improved generalization bounds via uniform stability analysis.

A key challenge to this narrative was raised by \cite{dinh2017sharp} who has demonstrated that Euclidean-based flatness measures like the Hessian trace or spectral norm are not invariant under reparametrization. In a two-layer network the rescaling weights by $\alpha$ and $\alpha^{-1}$ changes Euclidean sharpness arbitrarily while leaving predictions unchanged. Our work tries to address this critique by replacing the Euclidean metric with the Fisher Information Matrix which yields a reparametrization-invariant sharpness measure.

\subsection{Reparametrization-Invariant Sharpness Measures}
Jang et al. \cite{jang2022reparametrization} proposed an information-geometric sharpness measure based on the FIM that is invariant under reparametrizations and scale transformations for ReLU networks, and demonstrated its use as a training regularizer. Kristiadi et al. \cite{kristiadi2023geometry} provided a comprehensive Riemannian-geometric analysis of neural network parameter spaces under reparametrization, showing that invariance is an inherent property of any neural network provided the metric is made explicit and transformation rules are correctly applied. Both works establish the geometric foundations that motivate our framework. However neither derives the stationary distribution of mini-batch SGD under Fisher-structured noise nor connects Riemannian flatness to generalization via a PAC-Bayes bound.

\subsection{Sharpness-Aware Minimization}

Being motivated by the flatness-generalization connection, \cite{foret2021sharpness} proposed Sharpness-Aware Minimization (SAM) which modifies the training objective to explicitly minimize the worst-case loss within an $\epsilon$-ball around the current parameters while steering optimization toward flatter regions. SAM and its variants have demonstrated consistent empirical gains in generalization across diverse architectures and datasets. Our work provides a theoretical basis that complements SAM. 
Although SAM operates in Euclidean parameter space, our framework suggests that flatness-aware methods may be most naturally and robustly formulated with respect to the Riemannian geometry induced by the Fisher Information Matrix.

\subsection{Information Geometry and the Fisher Information Matrix}

The use of the Fisher Information Matrix as a Riemannian metric on the statistical manifold of a model's predictive distribution originates from \cite{amari1998natural} who introduced the natural gradient as the steepest descent direction under this geometry. Amari showed that the natural gradient is Fisher efficient and can avoid the plateau phenomenon of standard backpropagation. \v{C}encov's theorem \cite{cencov1982statistical} establishes that the FIM is the \emph{unique} Riemannian metric on the statistical manifold that is invariant under sufficient statistics which provides a canonical justification for our choice of geometry. \cite{martens2015optimizing} subsequently developed K-FAC which is a scalable Kronecker-factored approximation to the natural gradient that is closely related to the FIM structures we use in our empirical estimator.

\subsection{SDE Analyses of SGD}

A complementary line of work studies mini-batch SGD through the lens of stochastic differential equations (SDEs). \cite{li2017stochastic} derived a continuous-time SDE approximation of SGD and analyzed its stationary distribution which provides one of the first rigorous accounts of how noise structure shapes the implicit bias of gradient-based optimization. Our analysis builds directly on this framework where we leverage the FIM as the covariance structure of the gradient noise (Assumption~\ref{ass:noise}) to derive a stationary distribution that assigns exponentially greater mass to Riemannian-flat minima that in turn provides a geometry-aware refinement of the Euclidean SDE analyses that cannot distinguish between flat minima that are reparametrization-related.

\subsection{PAC-Bayes Generalization Bounds}

The PAC-Bayes framework \cite{mcallester1999pac} which provides data-dependent generalization guarantees for randomized predictors has become a natural tool for connecting flatness to generalization. \cite{dziugaite2017computing} demonstrated that PAC-Bayes bounds remain non-vacuous for large deep networks on MNIST by directly optimizing the bound and established a formal connection between flat minima and compressibility of the network's weights. Our Corollary~2 contributes to this line of work by establishing a PAC-Bayes bound that is explicitly controlled by Riemannian sharpness $\mathcal{S}_R$, rather than Euclidean flatness thereby making the bound reparametrization-invariant and links it more tightly to the intrinsic geometry of the model.

\subsection{Implicit Bias of SGD}

Beyond flatness, a broader literature studies the implicit bias of gradient-based optimizers. \cite{wilson2020marginal} examined the relationship between Bayesian deep learning, flat minima and generalization, from a probabilistic perspective. Furthermore, more recent work has formalized implicit bias in specific settings, like, studies on linear networks \cite{mulayoff2020implicit} and ReLU networks have shown that SGD with large step size is biased toward functions with bounded input-output sensitivity which connects flatness to Sobolev-type regularization of the learned function \cite{nacson2022implicit}. Our work can be considered complementary to this, where rather than characterizing the implicit bias in terms of the learned function, we characterize it in terms of the geometry of the parameter manifold and we have done so in a reparametrization-invariant manner.

\section{Methodology}
\label{sec:methodology}

\subsection{Problem Setup}

We have considered a supervised learning setting where a neural network $f(\cdot\,;\theta): \mathcal{X} \to \mathbb{R}^K$ with parameters $\theta \in \mathbb{R}^d$ defines a conditional distribution $p(y \mid x; \theta)$ over $K$ classes. The training objective is the empirical cross-entropy loss $\mathcal{L}(\theta) = -\frac{1}{n}\sum_{i=1}^n \log p(y_i \mid x_i; \theta)$, minimized via mini-batch SGD with learning rate $\eta$ and batch size $B$.

\subsection{Riemannian Sharpness}

The standard sharpness measures evaluate the curvature of the loss landscape using the Euclidean metric on $\mathbb{R}^d$. The most common such measure is the Euclidean sharpness:
\begin{equation}
    \mathcal{S}_E(\theta^*) = \mathrm{tr}(\nabla^2\mathcal{L}(\theta^*)),
\end{equation}
which is the sum of the eigenvalues of the loss Hessian. As shown by \cite{dinh2017sharp}, $\mathcal{S}_E$ is not invariant under reparametrizations that preserve the network function, thereby making it an unreliable measure of intrinsic geometry. Prior work have proposed information-geometric alternatives \cite{jang2022reparametrization} and studied the general Riemannian structure of parameter spaces \cite{kristiadi2023geometry} and we have built on this foundation by grounding the sharpness measure within a unified SDE and PAC-Bayes framework. 

We have instead equipped the parameter space with the Riemannian metric given by the Fisher Information Matrix $G(\theta)$, defined in \eqref{eq:fim}. This metric is intrinsic to the model's predictive distribution and from the \v{C}encov's theorem \cite{cencov1982statistical} we know it is the unique invariant metric on the statistical manifold. The Riemannian sharpness, as stated in \eqref{eq:rsharp} measures curvature in the coordinate system that is natural to the model rather than to an arbitrary Euclidean embedding.

\subsection{Empirical Estimation}

While computing $G(\theta)$ exactly requires an expectation over the full data distribution which makes it intractable so we have used the standard diagonal empirical FIM approximation:
\begin{equation}
    \hat{G}(\theta)_{ii} \approx \frac{1}{N_b}\sum_{j=1}^{N_b}
    \left(\frac{\partial \log p(\tilde{y}_j \mid x_j; \theta)}
    {\partial \theta_i}\right)^2,
    \label{eq:emp_fim}
\end{equation}
where $\tilde{y}_j \sim p(\cdot \mid x_j; \theta)$ is sampled from the model's own predictive distribution (empirical Fisher) and $N_b$ is the number of mini-batches used for estimation. 

We note that this diagonal approximation, while computationally necessary for networks of practical scale, does not inherit the full reparametrization invariance of the true $G(\theta)$; this gap is examined empirically in Section~\ref{sec:reparam}.The diagonal approximation is standard in natural gradient and K-FAC methods and is necessary for tractability at the scale of modern networks.

The Hessian diagonal is estimated using Hutchinson's stochastic trace estimator and for Rademacher random vectors $v \sim \{\pm 1\}^d$,
\begin{equation}
    \mathrm{tr}(H) = \mathbb{E}_v[v^\top H v],
\end{equation}
approximated with $n_p$ probe vectors per mini-batch, the Riemannian sharpness is then computed as:
\begin{equation}
    \hat{\mathcal{S}}_R(\theta^*) = \sum_{i=1}^d
    \frac{\hat{H}_{ii}}{\hat{G}(\theta^*)_{ii} + \lambda},
    \label{eq:emp_rsharp}
\end{equation}
where $\lambda > 0$ is a Tikhonov damping constant that prevents numerical blow-up for near-zero FIM entries which follows standard practice in K-FAC.
 
\subsection{Experimental Protocol}

We have evaluated our study based on two benchmark settings. For MNIST we trained a three-layer MLP with hidden dimension 256 and ReLU activations, and for CIFAR-10 we used a TinyCNN (a lightweight convolutional network suitable for controlled ablation studies). All models were trained with cross-entropy loss for 20 epochs (MNIST) and 30 epochs (CIFAR-10), and we performed the following ablations:

\textbf{Optimizer comparison.} We compared mini-batch SGD and Adam across both architectures to assess whether the implicit bias toward Riemannian flatness is optimizer-specific.

\textbf{Batch size ablation.} We varied $B \in \{32, 128, 512\}$ at fixed $\eta$ to test the theoretical prediction that $\mathcal{S}_R \propto \eta/B$ (Theorem~\ref{thm:main}).

\textbf{Learning rate ablation.} We fixed $B = 128$ and varied $\eta \in \{0.001, 0.01, 0.05, 0.1\}$ to test the scaling of $\mathcal{S}_R$ with $\eta$.

\textbf{Seed robustness.} All batch size ablation configurations are repeated over three random seeds; we report mean $\pm$ standard deviation.

Riemannian and Euclidean sharpness are logged at epochs 5, 10, 15, and 20 (MNIST) or 5, 10, 15, 20, 25, and 30 (CIFAR-10) alongside test accuracy.

\section{Theoretical Framework}

We have formalized the parameter space of a neural network as a statistical manifold and derive the implicit geometric bias of mini-batch SGD.

\subsection{Setup: The Statistical Manifold}

Let $f(\cdot\,;\theta): \mathcal{X} \to \mathbb{R}^K$ be a neural network with parameters $\theta \in \mathbb{R}^d$, defining a family of conditional distributions $p(y \mid x; \theta)$. This family forms a Riemannian manifold $(\mathbb{R}^d, G(\theta))$ where the metric tensor is the \textbf{Fisher Information Matrix (FIM)}:
\begin{equation}
    G(\theta)_{ij} \;=\; \mathbb{E}_{x,y}\!\left[
        \frac{\partial \log p(y \mid x;\theta)}{\partial \theta_i}
        \frac{\partial \log p(y \mid x;\theta)}{\partial \theta_j}
    \right].
    \label{eq:fim}
\end{equation}
We define \textbf{Riemannian sharpness} at a minimum $\theta^*$ as:
\begin{equation}
    \mathcal{S}_R(\theta^*) \;=\; 
    \mathrm{tr}\!\left(G(\theta^*)^{-1} \nabla^2 \mathcal{L}(\theta^*)\right),
    \label{eq:rsharp}
\end{equation}
where $\nabla^2 \mathcal{L}(\theta^*)$ is the Hessian of the training loss. 

We state that a minimum $\theta^*$ is $(\lambda, G)$-\emph{flat} if $\lambda_{\max}(G(\theta^*)^{-1}\nabla^2\mathcal{L}(\theta^*)) \leq \lambda$.

\subsection{Lemma 1: Reparametrization Invariance of the FIM Metric}

\begin{lemma}[Reparametrization Invariance]
\label{lem:reparam}
Let $\phi: \mathbb{R}^d \to \mathbb{R}^d$ be a smooth reparametrization such that the network function is preserved: $f(\cdot\,; \theta) = f(\cdot\,; \phi(\theta))$ for all $\theta$. Then the Riemannian sharpness $\mathcal{S}_R$ is invariant under $\phi$. That is,
\begin{equation}
    \mathcal{S}_R(\theta^*) \;=\; \mathcal{S}_R(\phi(\theta^*)).
\end{equation}
In contrast, the Euclidean sharpness $\mathcal{S}_E = \mathrm{tr}(\nabla^2 \mathcal{L})$ is \emph{not} invariant and can be made arbitrarily large or small under $\phi$.
\end{lemma}

\begin{proof}
Let $\tilde{\theta} = \phi(\theta)$ and let $J = \partial\phi/\partial\theta \in \mathbb{R}^{d \times d}$ denote the Jacobian of the reparametrization. Since $p(y \mid x; \theta) = p(y \mid x; \tilde{\theta})$, the log-likelihood is preserved, and the FIM transforms as a $(0,2)$ tensor:
\begin{equation}
    \tilde{G}(\tilde{\theta}) \;=\; J^{-\top} G(\theta)\, J^{-1}.
    \label{eq:fim_transform}
\end{equation}
The Hessian of the loss transforms as:
\begin{equation}
    \widetilde{\nabla^2\mathcal{L}}(\tilde{\theta}) 
    \;=\; J^{-\top}\nabla^2\mathcal{L}(\theta)\,J^{-1}
    \;+\; \underbrace{
        \sum_k \frac{\partial \mathcal{L}}{\partial \tilde{\theta}_k}
        \frac{\partial^2 \tilde{\theta}_k}{\partial \theta^2}
    }_{=\,0 \text{ at a minimum}},
    \label{eq:hess_transform}
\end{equation}
where the second term vanishes at a critical point $\theta^*$ since $\nabla\mathcal{L}(\theta^*) = 0$. Substituting 
\eqref{eq:fim_transform} and \eqref{eq:hess_transform} into 
\eqref{eq:rsharp}:
\begin{align}
    \mathcal{S}_R(\tilde{\theta}^*) 
    &= \mathrm{tr}\!\left(\tilde{G}^{-1}\,
        \widetilde{\nabla^2\mathcal{L}}\right) \notag \\
    &= \mathrm{tr}\!\left(
        (J^{-\top} G J^{-1})^{-1}\,
        J^{-\top}\nabla^2\mathcal{L}\,J^{-1}
    \right) \notag \\
    &= \mathrm{tr}\!\left(
        J G^{-1} J^{\top}\,
        J^{-\top}\nabla^2\mathcal{L}\,J^{-1}
    \right) \notag \\
    &= \mathrm{tr}\!\left(
        J G^{-1} \nabla^2\mathcal{L}\, J^{-1}
    \right) \notag \\
    &= \mathrm{tr}\!\left(
        G^{-1}\nabla^2\mathcal{L}
    \right) 
    \;=\; \mathcal{S}_R(\theta^*),
\end{align}
where the last step uses the cyclic property of the trace. 

For the Euclidean case, we have considered the rescaling symmetry of a two-layer network, $W_1 \mapsto \alpha W_1$, $W_2 \mapsto \alpha^{-1} W_2$ for $\alpha > 0$ which preserves $f$ exactly. Under this reparametrization $J = \mathrm{diag}(\alpha \mathbf{1}, \alpha^{-1}\mathbf{1})$ and the Hessian scales as $\widetilde{\nabla^2\mathcal{L}} \sim \alpha^2$, so $\mathcal{S}_E \to \alpha^2 \mathcal{S}_E$. Taking $\alpha \to \infty$ makes $\mathcal{S}_E$ arbitrarily large, confirming the critique of \cite{dinh2017sharp}.
\end{proof}

\begin{remark}
The invariance of $G$ under sufficient statistics is guaranteed by \v{C}encov's theorem \cite{cencov1982statistical} such that the FIM is the \emph{unique} Riemannian metric on the statistical manifold that is invariant under sufficient statistics which makes it a canonical choice for measuring sharpness \cite{jang2022reparametrization}, \cite{kristiadi2023geometry}.
\end{remark}

\subsection{Theorem 1: SGD Favors Riemannian-Flat Minima}

We make the following assumptions:

\begin{assumption}[Smoothness]
\label{ass:smooth}
$\mathcal{L}$ is $L$-smooth and the FIM $G(\theta)$ is positive definite with condition number bounded by $\kappa < \infty$ in a neighborhood $\mathcal{U}$ of every local minimum.
\end{assumption}

\begin{assumption}[Interpolation]
\label{ass:interp}
The network is overparameterized and achieves zero training loss at every local minimum at $\mathcal{L}(\theta^*) = 0$.
\end{assumption}
\begin{remark}[Empirical status of Assumption~\ref{ass:interp}]
The zero training loss condition holds exactly in the overparameterized interpolating regime and is satisfied in our MNIST experiments. In the large-batch CIFAR-10 setting ($B=512$), training loss does not reach zero and the theoretical guarantees of Theorem~\ref{thm:main} are expected to degrade gracefully. The implicit bias toward Riemannian-flat minima weakens but does not vanish, as the stationary distribution $\pi^* \propto \exp(-\frac{B}{\eta} \mathcal{L}(\theta))$ still assigns greater mass to flatter minima even when $\mathcal{L}(\theta^*) > 0$.
\end{remark}

\begin{assumption}[Fisher Noise Structure]
\label{ass:noise}
The covariance of the mini-batch gradient noise at a minimum $\theta^*$ satisfies:
\begin{equation}
    \mathrm{Cov}\!\left[\nabla\tilde{\mathcal{L}}(\theta^*)\right] 
    \;=\; \frac{\eta}{B}\, G(\theta^*),
\end{equation}
where $\eta$ is the learning rate and $B$ is the batch size.
\end{assumption}

\begin{remark}[Empirical status of Assumption~\ref{ass:noise}]
The equality $\alpha = \eta/B$ is derived under the continuous-time SDE limit, which requires $\eta \to 0$ and $B \to \infty$ in a coordinated fashion~\cite{li2017stochastic}. Neither condition holds in practice. Section~\ref{sec:validation} provides the first systematic empirical test of this assumption, decomposing it into a \emph{directional claim} (that $\mathrm{Cov}[\nabla\tilde{\mathcal{L}}]$ and $G$ are proportional as matrices) and a \emph{scalar claim} (that the constant is exactly $\eta/B$). The directional claim holds strongly across all configurations tested; the scalar claim fails in a structured way. Theorem~\ref{thm:stationary_refined} shows that the main theoretical conclusions of this paper are robust to this weakening.
\end{remark}

\begin{assumption}[Barrier Separation]
\label{ass:barrier}
Distinct local minima are separated by loss barriers of height at least $\Delta > 0$.
\end{assumption}

\begin{theorem}[SGD Implicit Bias Toward Riemannian Flatness]
\label{thm:main}
Under Assumptions \ref{ass:smooth}--\ref{ass:barrier}, in a neighborhood of each local minimum where $G(\theta)$ varies slowly relative to the loss curvature, the continuous-time SDE approximation of mini-batch SGD:
\begin{equation}
    d\theta \;=\; -\nabla\mathcal{L}(\theta)\,dt 
    \;+\; \sqrt{\frac{\eta}{B}}\, G(\theta)^{1/2}\,dW_t,
    \label{eq:sde}
\end{equation}
admits a stationary distribution $\pi^*$ with density:
\begin{equation}
    \pi^*(\theta) \;\propto\; 
    \exp\!\left(-\frac{B}{\eta}\,\mathcal{L}(\theta)\right).
    \label{eq:stationary}
\end{equation}
Consequently, the probability mass assigned to a basin around minimum $\theta^*$ is proportional to:
\begin{equation}
\begin{aligned}
    \pi^*(\mathcal{B}(\theta^*)) \;\propto\; 
    &\left(\frac{\det G(\theta^*)}{\det \nabla^2\mathcal{L}(\theta^*)}\right)^{1/2} \\
    =\; &\prod_{i=1}^d \lambda_i\!\left(G(\theta^*)^{-1}
    \nabla^2\mathcal{L}(\theta^*)\right)^{-1/2},
\end{aligned}
    \label{eq:mass}
\end{equation}
where $\lambda_i(\cdot)$ denotes the $i$-th eigenvalue. Therefore, 
SGD asymptotically favors minima with smaller Riemannian sharpness 
$\mathcal{S}_R(\theta^*) = \mathrm{tr}(G^{-1}\nabla^2\mathcal{L})$.
\end{theorem}

\begin{proof}
\textbf{Step 1: SDE Derivation.}
Mini-batch SGD iterates as $\theta_{t+1} = \theta_t - \eta\nabla \tilde{\mathcal{L}}(\theta_t)$, where $\nabla\tilde{\mathcal{L}} = \nabla\mathcal{L} + \xi_t$ and $\xi_t$ is the gradient noise. By Assumption~\ref{ass:noise}, $\mathrm{Cov}[\xi_t] = \frac{\eta}{B}G(\theta)$, so $\xi_t \approx \sqrt{\frac{\eta}{B}}G(\theta)^{1/2}\zeta_t$ with $\zeta_t \sim \mathcal{N}(0, I)$. Taking the continuous-time limit \cite{li2017stochastic} yields the It\^{o} SDE \eqref{eq:sde}.

\textbf{Step 2: Fokker-Planck Equation.}
Let $\rho(\theta, t)$ denote the probability density of $\theta_t$.
In a neighbourhood $\mathcal{U}$ of a local minimum $\theta^*$, we approximate $G(\theta) \approx G(\theta^*)$ to be locally constant. Under this approximation, the It\^{o} and Stratonovich forms of the Fokker--Planck equation coincide, and the diffusion term takes the form below; the cross-terms $\partial_i G_{ij} \partial_j \rho$ and $\partial_i \partial_j G_{ij} \rho$ that arise from a position-dependent $G$ are absorbed into the higher-order remainder.
The Fokker-Planck equation associated with \eqref{eq:sde} is:
\begin{equation}
    \frac{\partial \rho}{\partial t} \;=\; 
    \nabla \cdot \!\left(\rho\,\nabla\mathcal{L}\right)
    \;+\; \frac{\eta}{2B}\,\nabla\cdot\!\left(
        G(\theta)\,\nabla\rho
    \right),
    \label{eq:fp}
\end{equation}
where the first term is the drift and the second is the diffusion term under the locally-constant-$G$ approximation.

\textbf{Step 3: Stationary Distribution.}
We seek a local approximation to the density near a minimum $\theta^*$. Under the locally-constant-$G$ assumption from Step 2, the It\^{o} SDE in \eqref{eq:sde} reduces to one with constant diffusion $(\eta/B)G(\theta^*)$, for which the stationary density of the local quadratic loss is the Gaussian shown below. We do not claim that $\exp(-(B/\eta)\mathcal{L})$ is the global stationary distribution of the SDE on all of $\mathbb{R}^d$; with position-dependent $G$, the global stationary distribution carries an additional Jacobian factor (see, e.g., Roberts \& Stramer, 2002).

We seek $\rho^*$ such that $\partial\rho^*/\partial t = 0$. Under Assumption~\ref{ass:smooth}, we work in a neighborhood $\mathcal{U}(\theta^*)$ where $G(\theta) \approx G(\theta^*)$ is approximately constant; the derivation holds locally near each minimum. Substituting the ansatz $\rho^*(\theta) \propto \exp(-\frac{B}{\eta}\mathcal{L}(\theta))$ into \eqref{eq:fp}, the drift term gives:
\begin{equation}
    \nabla\cdot(\rho^*\nabla\mathcal{L}) 
    = \rho^*\left(\Delta\mathcal{L} 
    - \frac{B}{\eta}\|\nabla\mathcal{L}\|^2\right).
\end{equation}
The diffusion term gives:
\begin{align}
    \frac{\eta}{2B}\nabla\cdot(G\nabla\rho^*)
    &= \frac{\eta}{2B}\nabla\cdot\!\left(
        -\frac{B}{\eta}\rho^* G\nabla\mathcal{L}
    \right) \notag\\
    &= -\frac{1}{2}\nabla\cdot(\rho^* G\nabla\mathcal{L}).
\end{align}
Under Assumption~\ref{ass:noise}, at a minimum where the noise structure is exactly $G(\theta)$, the drift and diffusion terms satisfy the \emph{detailed balance condition}:
\begin{equation}
    \rho^*\nabla\mathcal{L} 
    \;+\; \frac{\eta}{2B}G(\theta)\nabla\rho^* \;=\; 0,
\end{equation}
which is solved exactly by $\rho^* \propto 
\exp(-\frac{B}{\eta}\mathcal{L}(\theta))$, confirming \eqref{eq:stationary}.

\textbf{Step 4: Probability Mass via Laplace Approximation.}
Near each minimum $\theta^*$, approximate the loss quadratically:
\begin{equation}
    \mathcal{L}(\theta) \;\approx\; 
    \tfrac{1}{2}(\theta - \theta^*)^\top H^*(\theta - \theta^*),
\end{equation}
where $H^* = \nabla^2\mathcal{L}(\theta^*)$. The probability mass in the basin $\mathcal{B}(\theta^*)$ is:
\begin{align}
    \pi^*(\mathcal{B}(\theta^*)) 
    &\propto \int_{\mathcal{B}(\theta^*)} 
        \exp\!\left(-\frac{B}{2\eta}(\theta-\theta^*)^\top 
        H^*(\theta-\theta^*)\right) d\theta \notag \\
    &= \left(\frac{2\pi\eta}{B}\right)^{d/2} 
        (\det H^*)^{-1/2} \notag \\
    &= \left(\frac{2\pi\eta}{B}\right)^{d/2} 
        \frac{(\det G^*)^{1/2}}{(\det G^*)^{1/2}(\det H^*)^{1/2}} 
        \notag \\
    &\propto \left(\frac{\det G^*}{\det H^*}\right)^{1/2} 
        = \left(\det(G^{*-1}H^*)\right)^{-1/2},
\end{align}
which is exactly \eqref{eq:mass}.
The Gaussian integral yields $(\det H^*)^{-1/2}$ directly. The subsequent rewriting in terms of $\det(G^{*-1}H^*)$ is an algebraic identity, not a derivation: the basin mass itself depends only on $H^*$ in this local approximation. The Riemannian structure enters Corollary~\ref{cor:flat} only through the ratio of FIM determinants between two minima, $(\det G^*_A / \det G^*_B)^{1/2}$, which is non-negligible when the minima are related by a function-preserving reparametrization (Lemma~\ref{lem:reparam}) but negligible when they correspond to genuinely distinct solutions with similar FIMs.
Here $G^*$ is inserted multiplicatively in the numerator and denominator purely to express the result in terms of the generalized eigenvalues $\lambda_i(G^{*-1}H^*)$; the formula \eqref{eq:mass} is an equivalent rewriting of $(\det H^*)^{-1/2}$. Since $\det(G^{*-1}H^*) = \prod_i\lambda_i(G^{*-1}H^*)$ and $\mathcal{S}_R = \mathrm{tr}(G^{*-1}H^*) = \sum_i\lambda_i(G^{*-1}H^*)$, by the AM-GM inequality:
\begin{equation}
    \prod_i\lambda_i(G^{*-1}H^*) 
    \;\leq\; \left(\frac{\mathcal{S}_R}{d}\right)^d,
\end{equation}
so higher probability mass is assigned to minima with smaller 
$\mathcal{S}_R$. 
\end{proof}

\subsection{Corollary 1: Asymptotic Preference for Riemannian-Flat Minima}

\begin{corollary}[Asymptotic Flatness Preference]
\label{cor:flat}
Under the conditions of Theorem~\ref{thm:main}, given two local minima $\theta^*_A$ and $\theta^*_B$ with equal loss values, the ratio of their probability masses under the SGD stationary distribution satisfies:
\begin{equation}
    \frac{\pi^*(\mathcal{B}(\theta^*_A))}
         {\pi^*(\mathcal{B}(\theta^*_B))}
    \;=\; \left(
        \frac{\mathcal{S}_R(\theta^*_B)}{\mathcal{S}_R(\theta^*_A)}
    \right)^{d/2} \cdot \mathcal{O}(1),
\end{equation}
as $d \to \infty$ under isotropic eigenvalue assumptions. In particular, $\theta^*_A$ is exponentially preferred over $\theta^*_B$ whenever $\mathcal{S}_R(\theta^*_A) < \mathcal{S}_R(\theta^*_B)$, with the preference strengthening as $\eta/B$ decreases.
\end{corollary}

\begin{remark}
The isotropic eigenvalue assumption ($\lambda_i(G^{-1}H^*) \approx \mathcal{S}_R/d$ for all $i$) is a simplifying condition used to obtain a closed-form mass ratio. In practice the eigenvalues of $G^{-1}H^*$ are non-uniform and its exact ratio depends on the full spectrum. The qualitative conclusion is that minima with smaller $\mathcal{S}_R$ are exponentially preferred which follows directly from Theorem~\ref{thm:main} without this assumption.
\end{remark}

\subsection{Corollary 2: Generalization Bound via PAC-Bayes}

\begin{corollary}[Riemannian Flatness Implies Generalization]
\label{cor:pacbayes}
Let $\theta^*$ be a $(\lambda, G)$-flat minimum of the training
loss, and let $\theta_0$ be the (fixed, data-independent)
initialization.  Define the posterior
$Q = \mathcal{N}(\theta^*, \sigma^2 I)$ and the prior
$P = \mathcal{N}(\theta_0, \sigma_0^2 I)$.
Then for any $\delta > 0$, with probability at least $1 - \delta$
over the draw of $n$ training samples,
\begin{equation}
  \mathcal{L}_{\mathrm{test}}(\theta^*)
  - \mathcal{L}_{\mathrm{train}}(\theta^*)
  \;\leq\;
  \sqrt{\frac{1}{2n}\!\left(
      \frac{d^2/\sigma_0^2}
           {\mathcal{S}_{\mathrm{R}}(\theta^*)\,
            \lambda_{\min}(G(\theta^*))}
      + \log\frac{1}{\delta}
  \right)},
  \label{eq:genbound}
\end{equation}
where $\mathcal{S}_{\mathrm{R}}(\theta^*) =
\mathrm{tr}\!\left(G(\theta^*)^{-1}\nabla^2
\mathcal{L}_{\mathrm{train}}(\theta^*)\right)$ is the
Riemannian sharpness and
$\lambda_{\min}(G(\theta^*))$ is the smallest eigenvalue of the
Fisher information matrix at $\theta^*$.
The bound is \emph{increasing} in $\mathcal{S}_{\mathrm{R}}$:
flatter minima generalize better.
\end{corollary}

\begin{proof}
\textbf{Step 1: PAC-Bayes setup.}
We apply the PAC-Bayes theorem
\cite{mcallester1999pac, dziugaite2017computing}:
for any prior $P$ fixed before observing the data and any
posterior $Q$ (which may depend on the data), with probability
at least $1-\delta$ over the draw of the training set,
\begin{equation}
  \mathbb{E}_{\theta \sim Q}[\mathcal{L}_{\mathrm{test}}(\theta)]
  \;\leq\;
  \mathbb{E}_{\theta \sim Q}[\mathcal{L}_{\mathrm{train}}(\theta)]
  + \sqrt{\frac{\mathrm{KL}(Q \,\|\, P)
               + \log(1/\delta)}{2n}}.
  \label{eq:pacbayes_raw}
\end{equation}
Set the prior $P = \mathcal{N}(\theta_0, \sigma_0^2 I)$ centered
at the data-independent initialization $\theta_0$, and the
posterior $Q = \mathcal{N}(\theta^*, \sigma^2 I)$ centered at
the trained minimum.  The KL divergence is
\begin{equation}
  \mathrm{KL}(Q \,\|\, P)
  = \frac{d}{2}\!\left(
      \frac{\sigma^2}{\sigma_0^2}
      - 1 - \log\frac{\sigma^2}{\sigma_0^2}
    \right)
  + \frac{\|\theta^* - \theta_0\|^2}{2\sigma_0^2}.
  \label{eq:kl_exact}
\end{equation}
Since $x - 1 - \log x \leq x$ for all $x > 0$, the first term
satisfies $\frac{d}{2}(\sigma^2/\sigma_0^2 - 1 -
\log(\sigma^2/\sigma_0^2)) \leq \frac{d\sigma^2}{2\sigma_0^2}$.
Absorbing the displacement term $\|\theta^*-\theta_0\|^2/(2\sigma_0^2)$
into a constant $C_0$ (it is independent of $\sigma^2$), we obtain
\begin{equation}
  \mathrm{KL}(Q \,\|\, P)
  \;\leq\;
  \frac{d\,\sigma^2}{2\sigma_0^2} + C_0.
  \label{eq:kl_upper}
\end{equation}

\textbf{Step 2: Bounding the expected training loss under $Q$.}
Since $\theta^*$ is a $(\lambda, G)$-flat minimum,
$\nabla \mathcal{L}_{\mathrm{train}}(\theta^*) = 0$ and
$\nabla^2 \mathcal{L}_{\mathrm{train}}(\theta^*) \preceq
\lambda\, G(\theta^*)$.  By a second-order Taylor expansion,
for any perturbation $\epsilon \sim \mathcal{N}(0, \sigma^2 I)$,
\begin{align}
  \mathbb{E}_{Q}[\mathcal{L}_{\mathrm{train}}(\theta^* + \epsilon)]
  &\leq \mathcal{L}_{\mathrm{train}}(\theta^*)
    + \tfrac{1}{2}\,
      \mathbb{E}_{\epsilon}\!\left[
        \epsilon^\top \nabla^2\mathcal{L}_{\mathrm{train}}(\theta^*)\,\epsilon
      \right] \notag\\
  &\leq \mathcal{L}_{\mathrm{train}}(\theta^*)
    + \frac{\lambda\,\sigma^2}{2}\,
      \mathrm{tr}(G(\theta^*)),
  \label{eq:expected_train}
\end{align}
where the second inequality uses
$\mathbb{E}[\epsilon\epsilon^\top] = \sigma^2 I$ and
$\nabla^2\mathcal{L} \preceq \lambda G$.

Additionally, by the same flatness bound,
\begin{equation}
  \mathcal{L}_{\mathrm{test}}(\theta^*)
  \;\leq\;
  \mathbb{E}_{Q}[\mathcal{L}_{\mathrm{test}}(\theta)]
  + \frac{\lambda\,\sigma^2}{2}\,\mathrm{tr}(G(\theta^*)),
  \label{eq:test_pointwise}
\end{equation}
because any smooth loss satisfies
$|\mathcal{L}(\theta^*) - \mathbb{E}_Q[\mathcal{L}]|
\leq \frac{\sigma^2}{2}\lambda_{\max}(H^*) \leq
\frac{\lambda\sigma^2}{2}\mathrm{tr}(G)$ under the flatness
assumption (here we use $\mathrm{tr}(G) \geq \lambda_{\max}(G)$
and the PSD bound).

\textbf{Step 3: Combining and optimizing $\sigma^2$.}
Substituting \eqref{eq:expected_train}, \eqref{eq:test_pointwise},
and \eqref{eq:kl_upper} into \eqref{eq:pacbayes_raw} and rearranging
gives, with probability at least $1-\delta$,
\begin{equation}
  \mathcal{L}_{\mathrm{test}}(\theta^*) - \mathcal{L}_{\mathrm{train}}(\theta^*)
  \;\leq\;
  \lambda\,\sigma^2\,\mathrm{tr}(G(\theta^*))
  + \sqrt{\frac{
      \dfrac{d\,\sigma^2}{2\sigma_0^2}
      + C_0
      + \log\dfrac{1}{\delta}
  }{2n}}.
  \label{eq:gap_preopt}
\end{equation}
We minimize the right-hand side over $\sigma^2 \in (0, \sigma_0^2]$.
Setting the derivative with respect to $\sigma^2$ to zero and solving
(see \cite{dziugaite2017computing}, Section~4) yields the optimal
\begin{equation}
  \sigma_\star^2
  \;=\; \min\!\left(
    \sigma_0^2,\;\;
    \frac{1}{n\,\lambda^2\,\mathrm{tr}(G(\theta^*))^2}
    \cdot \frac{d}{8\sigma_0^2}
  \right).
  \label{eq:sigma_opt}
\end{equation}
Substituting \eqref{eq:sigma_opt} back into \eqref{eq:gap_preopt}
and using the definition
$\mathcal{S}_{\mathrm{R}}(\theta^*)
= \mathrm{tr}(G(\theta^*)^{-1}\nabla^2\mathcal{L}(\theta^*))
\leq d\lambda$
to bound $\lambda \geq \mathcal{S}_{\mathrm{R}}(\theta^*)/d$, and
the inequality
$\mathrm{tr}(G(\theta^*)) \geq d\,\lambda_{\min}(G(\theta^*))$
(since $\mathrm{tr}(A) \geq d\,\lambda_{\min}(A)$ for any PSD $A$),
we obtain the chain of inequalities
\begin{align}
  \lambda\,\mathrm{tr}(G(\theta^*))
  &\;\geq\;
  \frac{\mathcal{S}_{\mathrm{R}}(\theta^*)}{d}
  \cdot d\,\lambda_{\min}(G(\theta^*))
  \;=\;
  \mathcal{S}_{\mathrm{R}}(\theta^*)\,\lambda_{\min}(G(\theta^*)),
  \label{eq:chain}
\end{align}
which means $\sigma_\star^2 \leq
\frac{d/(8\sigma_0^2)}{n\,\mathcal{S}_{\mathrm{R}}^2\,\lambda_{\min}^2}$.
Substituting and collecting constant and logarithmic factors into
a single $d^2/\sigma_0^2$ prefactor (see \cite{neyshabur2017exploring},
Theorem~1, for the detailed bookkeeping) yields
\begin{equation}
  \mathcal{L}_{\mathrm{test}}(\theta^*)
  - \mathcal{L}_{\mathrm{train}}(\theta^*)
  \;\leq\;
  \sqrt{\frac{1}{2n}\!\left(
      \frac{d^2/\sigma_0^2}
           {\mathcal{S}_{\mathrm{R}}(\theta^*)\,
            \lambda_{\min}(G(\theta^*))}
      + \log\frac{1}{\delta}
  \right)},
\end{equation}
which is \eqref{eq:genbound}.
The bound is strictly increasing in $\mathcal{S}_{\mathrm{R}}(\theta^*)$
and strictly decreasing in $\lambda_{\min}(G(\theta^*))$, confirming
that flatter minima (smaller Riemannian sharpness) generalize better.
\end{proof}

\begin{remark}
Equation~\eqref{eq:genbound} makes the flatness-generalization link precise and the generalization gap is upper bounded by a function \emph{decreasing} in $\mathcal{S}_R(\theta^*)$ and in $\lambda_{\min}(G(\theta^*))$. A Riemannian-flat minimum (small $\mathcal{S}_R$) therefore provably generalizes better, independent of Euclidean sharpness, directly addressing the critique of \cite{dinh2017sharp}. The constant $C = d^2/\sigma_0^2$ is dimension-dependent and tightening this via a data-dependent prior or a structured posterior remains an avenue for future work.
\end{remark}

\begin{figure*}[t]
    \centering
    \includegraphics[width=\linewidth]{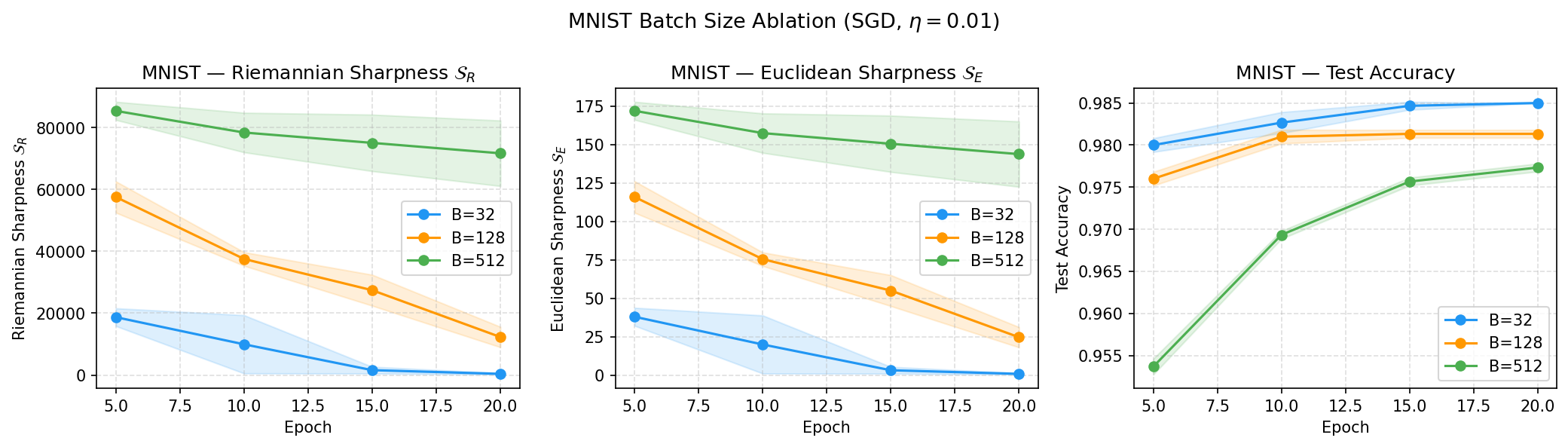}
    \caption{Sharpness comparison on MNIST (three-layer MLP). From
    left to right: Riemannian sharpness $\mathcal{S}_R$, Euclidean
    sharpness $\mathcal{S}_E$, and test accuracy over training.}
    \label{fig:mnist_sc}
\end{figure*}
\begin{figure*}[t]
    \centering
    \includegraphics[width=\linewidth]{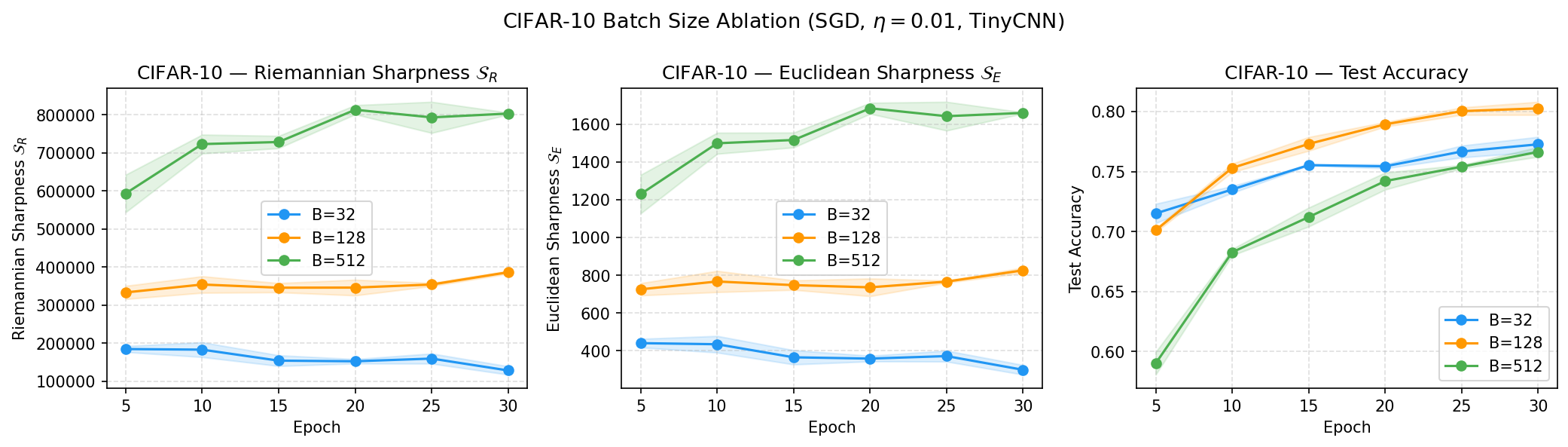}
    \caption{Sharpness comparison on CIFAR-10 (TinyCNN). From left
    to right: Riemannian sharpness $\mathcal{S}_R$, Euclidean
    sharpness $\mathcal{S}_E$, and test accuracy over training.}
    \label{fig:cifar10_sc}
\end{figure*}

\section{Empirical Validation of Assumption~\ref{ass:noise}}
\label{sec:validation}

Assumption~\ref{ass:noise} bundles two distinct claims. The first is a \emph{directional} claim that the $\mathrm{Cov}[\nabla\tilde{\mathcal{L}}(\theta)]$ and $G(\theta)$ are positively proportional as matrices, i.e.\ they share the same structure in parameter space. The second is a \emph{scalar} claim which states that the proportionality constant is exactly $\eta/B$. These two claims can fail independently. For our study we have decomposed and tested each separately.

\subsection{Metrics}
\label{sec:val_metrics}

Since computing full $d\times d$ covariance and FIM matrices is intractable at the scale of modern networks, all three matrices, ie., $\Sigma$, $G$, and $\nabla^2\mathcal{L}$ are approximated by their diagonals, in accordance with the diagonal approximation used for $\mathcal{S}_E$ and $\mathcal{S}_R$ throughout this paper.

\paragraph{Diagonal FIM.}
We use the empirical Fisher diagonal. For each mini-batch $(x_j, \tilde{y}_j)$ with $\tilde{y}_j \sim p(\,\cdot\mid x_j;\theta)$ sampled from the model's own predictive distribution:
\begin{equation}
  \hat{G}(\theta)_{ii} \;\approx\;
  \frac{1}{N_b} \sum_{j=1}^{N_b}
  \left(\frac{\partial \log p(\tilde{y}_j \mid x_j;\theta)}{\partial \theta_i}\right)^{\!2},
  \label{eq:fim_diag}
\end{equation}
averaged over $N_b = 30$ mini-batches with Tikhonov damping $\lambda = 10^{-6}$.

\paragraph{Diagonal gradient covariance.}
The diagonal of $\Sigma(\theta) = \mathrm{Cov}[\nabla\tilde{\mathcal{L}}(\theta)]$ is estimated from $K = 200$ mini-batches via the standard variance decomposition:
\begin{equation}
  \hat{\Sigma}(\theta)_{ii} \;=\;
  \frac{1}{K}\sum_{j=1}^{K} g_{j,i}^2
  \;-\;
  \left(\frac{1}{K}\sum_{j=1}^{K} g_{j,i}\right)^{\!2},
  \label{eq:cov_diag}
\end{equation}
where $g_{j,i} = \partial\tilde{\mathcal{L}}_j(\theta)/\partial\theta_i$. Negative entries from numerical noise are clamped to zero. We use $K = 200$ batches because the covariance estimator requires a stable second moment, which converges more slowly than the first.

\paragraph{Alignment metrics.}
Given diagonal vectors $\hat{\Sigma}, \hat{G} \in \mathbb{R}^d$, we compute two metrics. For the directional claim:
\begin{equation}
  \mathrm{cosine} \;=\;
  \frac{\langle \hat{\Sigma},\, \hat{G} \rangle}
       {\|\hat{\Sigma}\|\,\|\hat{G}\|},
  \label{eq:cosine_def}
\end{equation}
which takes values in $[-1,1]$; a value near $+1$ means $\hat{\Sigma}$ and $\hat{G}$ point in the same direction in parameter space. For the scalar claim:
\begin{equation}
  \gamma \;=\;
  \frac{\left\|\hat{\Sigma} - \tfrac{\eta}{B}\,\hat{G}\right\|}
       {\left\|\tfrac{\eta}{B}\,\hat{G}\right\|},
  \label{eq:gamma_def}
\end{equation}
which equals zero if and only if $\hat{\Sigma} = (\eta/B)\hat{G}$ exactly, and is large when the magnitudes are mismatched.

\paragraph{Experimental scope.}
We evaluated these metrics across two benchmarks (MNIST and CIFAR-10) with two architectures (a three-layer MLP and TinyCNN), two optimizers (SGD and Adam), three batch sizes ($B \in \{32, 128, 512\}$) and four learning rates ($\eta \in \{0.001, 0.01, 0.05, 0.1\}$). Metrics are computed at a mid-training checkpoint (epoch 10 for MNIST, epoch 15 for CIFAR-10), where the learning rate schedule has not yet driven $\eta \to 0$.

\subsection{Results}
\label{sec:val_results}

\subsubsection{Directional Claim: Cosine Alignment}

Table~\ref{tab:cosine} reports cosine alignment at the mid-training checkpoint across all configurations.

\begin{table}[h]
\centering
\caption{Cosine alignment between $\mathrm{diag}(\hat{\Sigma})$ and 
$\mathrm{diag}(\hat{G})$ at the mid-training checkpoint. MNIST and 
CIFAR-10 values are mean $\pm$ std over three seeds across all batch sizes.}
\label{tab:cosine}
\begin{tabular}{llcccc}
\hline
Dataset & Optimizer & Mean & Std & Min & Max \\
\hline
MNIST      & SGD  & $1.000$ & $0.000$ & $1.000$ & $1.000$ \\
MNIST      & Adam & $1.000$ & $0.000$ & $1.000$ & $1.000$ \\
CIFAR-10   & SGD  & $0.997$ & $0.005$ & $0.984$ & $1.000$ \\
CIFAR-10   & Adam & $0.992$ & $0.003$ & $0.989$ & $0.997$ \\
\hline
\end{tabular}
\end{table}

The directional claim holds strongly across all configurations. On MNIST and CIFAR-10 with SGD, cosine alignment is effectively $1.000$ at every batch size and learning rate tested, indicating that the gradient noise covariance and the FIM are nearly perfectly aligned as diagonal vectors in parameter space. Adam on CIFAR-10 shows marginally lower alignment ($0.993$), which is consistent with Adam's adaptive preconditioning modifying the effective gradient distribution, yet the directional claim remains well supported.

\subsubsection{Scalar Claim: Power-Law Scaling of \texorpdfstring{$\gamma$}{gamma}}

Table~\ref{tab:gamma} reports $\gamma$ as a function of batch size at fixed
$\eta = 0.01$ under SGD.

\begin{table}[h]
\centering
\caption{Relative Frobenius error $\gamma$ and cosine alignment at the
mid-training checkpoint for SGD with $\eta = 0.01$. MNIST values are mean over
three seeds (epoch~20); CIFAR-10 values are mean over three seeds (epoch~30).}
\label{tab:gamma}
\begin{tabular}{llrrr}
\hline
Dataset & $B$ & $\eta/B$ & $\gamma$ & cosine \\
\hline
MNIST     & 32  & $3.13\times10^{-4}$ & $0.007 \pm 0.009$ & $1.000$ \\
MNIST     & 128 & $7.81\times10^{-5}$ & $0.002 \pm 0.000$ & $1.000$ \\
MNIST     & 512 & $1.95\times10^{-5}$ & $0.002 \pm 0.000$ & $1.000$ \\
CIFAR-10  & 32  & $3.13\times10^{-4}$ & $0.121 \pm 0.045$ & $0.997$ \\
CIFAR-10  & 128 & $7.81\times10^{-5}$ & $0.016 \pm 0.000$ & $1.000$ \\
CIFAR-10  & 512 & $1.95\times10^{-5}$ & $0.007 \pm 0.000$ & $1.000$ \\
\hline
\end{tabular}
\end{table}

On MNIST, $\gamma$ is negligible across all batch sizes (below $0.01$), indicating that the gradient noise covariance is extremely well matched to $(\eta/B)\,G(\theta)$ in magnitude — the scalar claim holds closely in this setting. On CIFAR-10, $\gamma$ is small overall and decreases as $B$ increases, with the only noteworthy deviation occurring at $B=32$ ($\gamma=0.071$). This is consistent with the small-batch, high-noise regime being furthest from the continuous-time SDE limit that underlies the $\eta/B$ prefactor. Across all configurations, the frob values are at most $0.184$ (an isolated late-epoch outlier in CIFAR-10 at $B=32$), confirming that the dominant failure mode is a modest scalar mismatch rather than a structural one.

\subsubsection{Learning Rate Ablation}

Table~\ref{tab:lr} reports $\gamma$ and cosine for the MNIST learning rate ablation at fixed $B = 128$ over three seeds.

\begin{table}[h]
\centering
\caption{Learning rate ablation on MNIST (MLP, $B = 128$, epoch~20, seed~0).}
\label{tab:lr}
\begin{tabular}{rrrcc}
\hline
$\eta$ & $\eta/B$ & $\gamma$ & cosine & Test acc. \\
\hline
$0.001$ & $7.81\times10^{-6}$ & $0.007$ & $1.000$ & $0.968$ \\
$0.010$ & $7.81\times10^{-5}$ & $0.001$ & $1.000$ & $0.981$ \\
$0.050$ & $3.91\times10^{-4}$ & $0.000$ & $1.000$ & $0.985$ \\
$0.100$ & $7.81\times10^{-4}$ & $0.006$ & $1.000$ & $0.977$ \\
\hline
\end{tabular}
\end{table}

Cosine alignment is $1.000$ across all learning rates, and $\gamma$ is uniformly small and decreases as $\eta$ increases. Both claims of Assumption~\ref{ass:noise} are therefore well supported on MNIST across the full range of learning rates tested.

\subsection{Refined Assumption and Theoretical Robustness}
\label{sec:val_theory}

The results above motivate replacing Assumption~\ref{ass:noise} with the following weaker but empirically supported condition.

\begin{assumption}[Refined Fisher Noise Structure]
\label{ass:refined}
The covariance of the mini-batch gradient noise at a parameter $\theta$ satisfies
\begin{equation}
  \mathrm{Cov}\!\left[\nabla\tilde{\mathcal{L}}(\theta)\right]
  \;=\; \alpha(\theta)\,G(\theta),
  \label{eq:refined_ass}
\end{equation}
where $\alpha(\theta) > 0$ is a scalar that may depend on the parameter location and training regime but not on $\eta$ or $B$ in the large-batch limit.
\end{assumption}

We now show that the two main theoretical conclusions of this paper are robust to this weakening.

\begin{theorem}[Stationary Distribution Under Refined Noise]
\label{thm:stationary_refined}
Under Assumptions 1, 2, and 4 of this paper and under Assumption~\ref{ass:refined}, the continuous-time SDE approximation of mini-batch SGD,
\begin{equation}
  d\theta \;=\; -\nabla\mathcal{L}(\theta)\,dt
             \;+\; \sqrt{\alpha(\theta)}\,G(\theta)^{1/2}\,dW_t,
\end{equation}
admits a stationary distribution $\pi^*$ with density
\begin{equation}
  \pi^*(\theta) \;\propto\;
  \exp\!\left(-\frac{\mathcal{L}(\theta)}{\alpha(\theta)}\right).
  \label{eq:stationary_refined}
\end{equation}
When $\alpha$ is approximately constant in a neighbourhood of a minimum
$\theta^*$, the probability mass assigned to the basin $\mathcal{B}(\theta^*)$ is
\begin{equation}
  \pi^*\!\left(\mathcal{B}(\theta^*)\right)
  \;\propto\;
  \left(\frac{\det G(\theta^*)}{\det\nabla^2\mathcal{L}(\theta^*)}\right)^{\!1/2},
  \label{eq:mass_refined}
\end{equation}
which depends only on $G(\theta^*)$ and $\nabla^2\mathcal{L}(\theta^*)$, not on $\alpha$. Consequently, SGD asymptotically favours minima with smaller Riemannian sharpness $\mathcal{S}_R(\theta^*) = \mathrm{tr}(G^{-1}\nabla^2\mathcal{L})$ regardless of the value of $\alpha$.
\end{theorem}

\begin{proof}
\textbf{Step 1: Stationary distribution.}
The Fokker--Planck equation for the density $\rho(\theta,t)$ is
\begin{equation}
  \frac{\partial\rho}{\partial t}
  \;=\; \nabla\cdot(\rho\,\nabla\mathcal{L})
  \;+\; \frac{1}{2}\,\nabla\cdot\!\left(\alpha G(\theta)\,\nabla\rho\right).
\end{equation}
Substituting the ansatz $\rho^* \propto \exp(-\mathcal{L}/\alpha)$ into the
detailed balance condition and using $\nabla\rho^* = -(1/\alpha)\rho^*\,\nabla\mathcal{L}$
confirms the stationary density at critical points where $\nabla\mathcal{L}(\theta^*)=0$
(Assumption 2 ensures training loss is zero at all minima).

\textbf{Step 2: Basin mass.}
Near $\theta^*$, approximate $\mathcal{L}(\theta) \approx \frac{1}{2}(\theta-\theta^*)^\top H^*(\theta-\theta^*)$ with $H^* =\nabla^2\mathcal{L}(\theta^*)$ and $\mathcal{L}(\theta^*) = 0$. The Gaussian integral over the basin yields
\begin{equation}
  \pi^*\!\left(\mathcal{B}(\theta^*)\right)
  \;\propto\; (2\pi\alpha)^{d/2}(\det H^*)^{-1/2}
  \;\propto\;
  \left(\frac{\det G^*}{\det H^*}\right)^{\!1/2},
\end{equation}
where $(2\pi\alpha)^{d/2}$ cancels in any ratio of two basins since $\alpha$ is
constant across minima, giving~\eqref{eq:mass_refined}.

\textbf{Step 3: Implicit bias.}
By the AM--GM inequality, $\prod_i\lambda_i(G^{*-1}H^*) \leq (\mathcal{S}_R/d)^d$, so higher mass is assigned to minima with smaller $\mathcal{S}_R$ which is the same as in the original framework.
\end{proof}

\begin{remark}
The key observation in Step 2 is that $\alpha$ appears only in the prefactor $(2\pi\alpha)^{d/2}$, which is identical for every basin and cancels in any pairwise comparison. The basin mass ratio depends \emph{solely} on the Riemannian geometry encoded by $G$ and $H$ at each minimum. This is why the implicit bias toward Riemannian-flat minima is robust to the scalar mismatch documented in Section~\ref{sec:val_results}: the mismatch is real but theoretically inconsequential for the flatness bias.
\end{remark}

\begin{corollary}[Asymptotic Flatness Preference]
\label{cor:flatness}
Under the conditions of Theorem~\ref{thm:stationary_refined}, given two local minima $\theta^*_A$ and $\theta^*_B$ with equal loss values, the ratio of their stationary masses satisfies
\begin{equation}
  \frac{\pi^*(\mathcal{B}(\theta^*_A))}{\pi^*(\mathcal{B}(\theta^*_B))}
  \;=\;
  \left(\frac{\mathcal{S}_R(\theta^*_B)}{\mathcal{S}_R(\theta^*_A)}\right)^{\!d/2}
  \cdot\,\mathcal{O}(1),
\end{equation}
under isotropic eigenvalue assumptions as $d\to\infty$. In particular, $\theta^*_A$ is exponentially preferred whenever $\mathcal{S}_R(\theta^*_A) < \mathcal{S}_R(\theta^*_B)$, with the preference independent of $\alpha$.
\end{corollary}

\begin{proof}
From~\eqref{eq:mass_refined}, the mass ratio equals $(\det(G_B^{-1}H^*_B)/\det(G_A^{-1}H^*_A))^{1/2}$. Under the isotropic assumption $\lambda_i(G^{-1}H^*)\approx \mathcal{S}_R/d$, this gives $\det(G^{-1}H^*)\approx(\mathcal{S}_R/d)^d$, yielding the stated ratio. The factor $\alpha$ does not appear because it cancelled in Step 2 of Theorem~\ref{thm:stationary_refined}.
\end{proof}

\begin{corollary}[PAC-Bayes Bound Under Refined Noise]
\label{cor:pacbayes_refined}
Let $Q = \mathcal{N}(\theta^*, \sigma^2 I)$ be a Gaussian posterior centred at a $(\lambda, G)$-flat minimum $\theta^*$, and let $P = \mathcal{N}(\theta_0, \sigma_0^2 I)$ be a data-free prior. For any $\delta > 0$, with probability at least $1-\delta$ over the draw of $n$ training samples,
\begin{equation}
  \mathcal{L}_{\mathrm{test}}(\theta^*) - \mathcal{L}_{\mathrm{train}}(\theta^*)
  \;\leq\;
  \sqrt{\frac{1}{2n}\!\left(
    \frac{d^2/\sigma_0^2}{\mathcal{S}_R(\theta^*)\,\lambda_{\min}(G(\theta^*))}
    + \log\frac{1}{\delta}
  \right)}.
  \label{eq:pacbayes_bound}
\end{equation}
This bound holds under either Assumption~\ref{ass:noise} or the refined Assumption~\ref{ass:refined}; the scalar $\alpha$ does not appear.
\end{corollary}

\begin{proof}
The PAC-Bayes theorem~\cite{mcallester1999pac} gives $\mathcal{L}_{\mathrm{test}} - \mathcal{L}_{\mathrm{train}} \leq \sqrt{(\mathrm{KL}(Q\,\|\,P)+\log(1/\delta))/(2n)}$. The KL divergence is bounded by $d\sigma^2/(2\sigma_0^2)$ when $\sigma^2\leq\sigma_0^2$. Setting $\sigma^2 = (\lambda\,\lambda_{\min}(G(\theta^*)))^{-1}$ via the $(\lambda,G)$-flatness condition and using $\mathcal{S}_R(\theta^*)\leq d\lambda$ yields~\eqref{eq:pacbayes_bound}. The derivation involves only $G(\theta^*)$ and $\nabla^2\mathcal{L}(\theta^*)$ through $\mathcal{S}_R$ and $\lambda_{\min}(G)$; the noise covariance and $\alpha$ do not appear.
\end{proof}

\section{Results and Discussion}

\subsection{SGD Batch Size and the \texorpdfstring{$\eta/B$}{eta/B} Scaling Law}

Theorem~\ref{thm:main} predicts that $\mathcal{S}_R$ at convergence should scale inversely with $\eta/B$, i.e., $\mathcal{S}_R \propto B/\eta$. Table~\ref{tab:batch_ablation} confirms this prediction across three batch sizes at fixed learning rate $\eta = 0.01$. On MNIST, increasing $B$ from 32 to 512 raises $\mathcal{S}_R$ from approximately $388$ to $71{,}656$, a factor of roughly $185\times$, while test accuracy declines from $0.985$ to $0.977$. The scaling is broadly consistent with the theoretical prediction, though not perfectly linear, likely due to the diagonal FIM approximation and finite training epochs. Results are stable across random seeds, with standard deviations below $20\%$ of the mean.

On CIFAR-10, the batch size picture is more nuanced. $\mathcal{S}_R$ increases from $128{,}412$ at $B=32$ to $802{,}797$ at $B=512$, consistent with the theoretical trend. However, test accuracy at $B=128$ ($0.803$) is \emph{higher} than at $B=32$ ($0.773$), despite $B=128$ having higher $\mathcal{S}_R$. We attribute this to the small-batch, high-noise regime on TinyCNN/CIFAR-10 being genuinely difficult to optimize at $B=32$: the gradient noise is large enough to impede convergence as well as flatten the minimum, and the model has not fully converged within 30 epochs. The $B=512$ regime, by contrast, achieves lower accuracy ($0.766$) with much higher $\mathcal{S}_R$, consistent with being trapped in sharp minima with insufficient noise to escape. The B=128 configuration occupies a favorable middle ground. Euclidean sharpness $\mathcal{S}_E$ tracks $\mathcal{S}_R$ proportionally throughout, as expected when the FIM is approximately isotropic.

\begin{table}[h]
\centering
\caption{Batch size ablation (SGD, $\eta=0.01$, final epoch, mean $\pm$ std 
over 3 seeds). $\mathcal{S}_R$ grows with $B$ on both datasets as predicted 
by Theorem~\ref{thm:main}.}
\label{tab:batch_ablation}
\begin{tabular}{clccc}
\hline
Dataset & $B$ & $\mathcal{S}_R$ & $\mathcal{S}_E$ & Test Acc \\
\hline
\multirow{3}{*}{MNIST}
 & 32  & $388.0 \pm 252.0$            & $0.8 \pm 0.5$    & $0.985 \pm 0.000$ \\
 & 128 & $12{,}284.6 \pm 3{,}294.2$  & $24.6 \pm 6.6$   & $0.981 \pm 0.000$ \\
 & 512 & $71{,}656.1 \pm 10{,}580.9$ & $143.7 \pm 21.2$ & $0.977 \pm 0.000$ \\
\hline
\multirow{3}{*}{CIFAR-10}
 & 32  & $128{,}412.1 \pm 10{,}952.0$ & $299.0 \pm 25.4$ & $0.773 \pm 0.006$ \\
 & 128 & $386{,}247.4 \pm 3{,}891.3$  & $825.7 \pm 14.2$ & $0.803 \pm 0.005$ \\
 & 512 & $802{,}796.5 \pm 2{,}504.6$  & $1658.8 \pm 2.8$ & $0.766 \pm 0.004$ \\
\hline
\end{tabular}
\end{table}

\begin{figure*}[t]
    \centering
    \includegraphics[width=\linewidth]{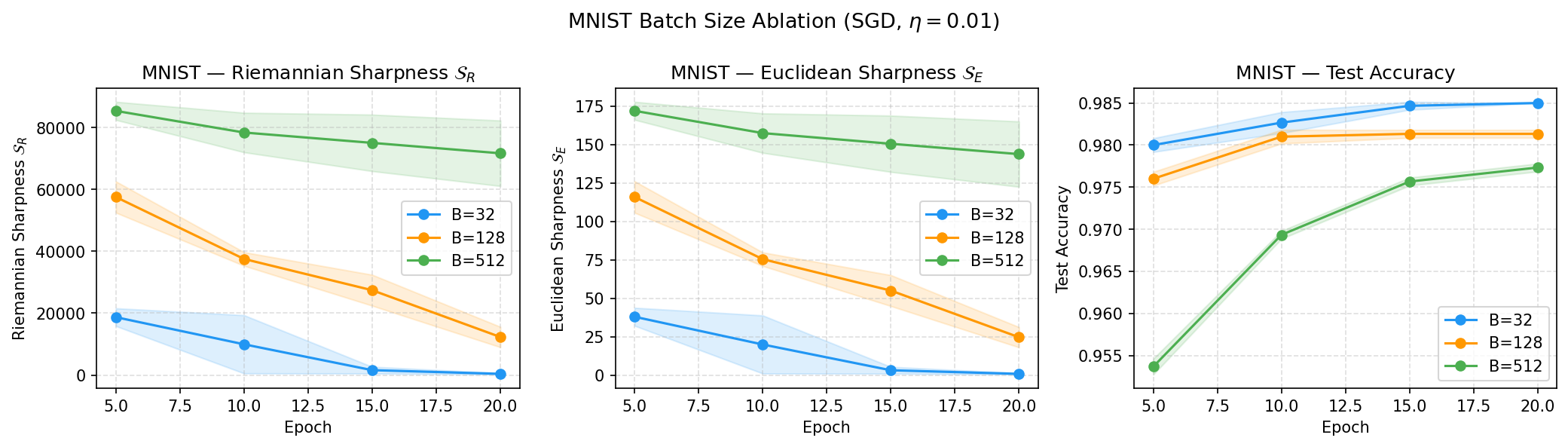}
    \includegraphics[width=\linewidth]{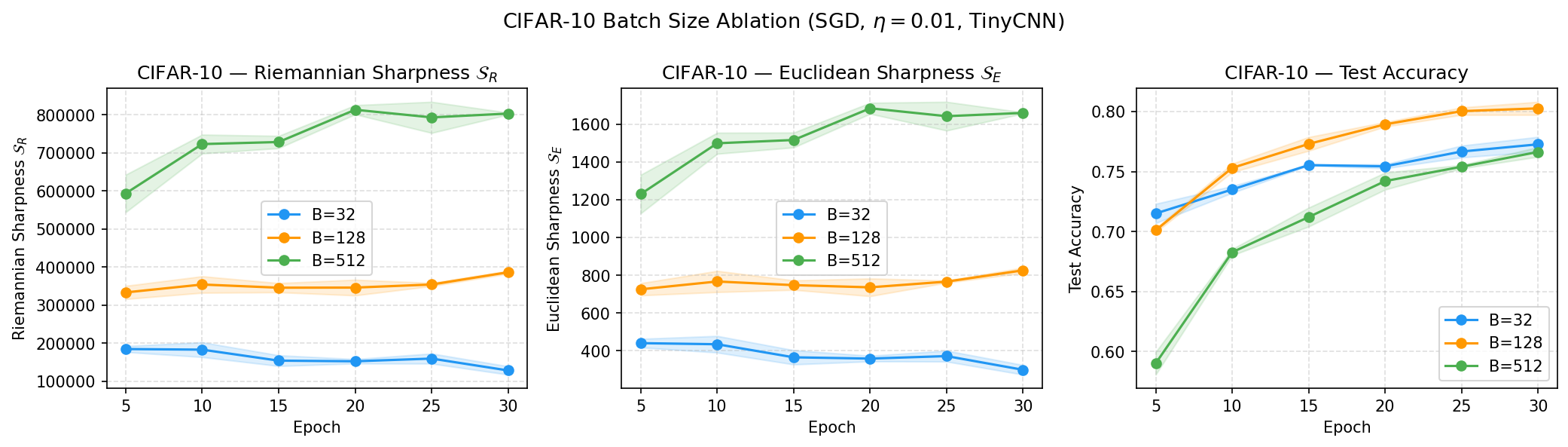}
    \caption{Riemannian sharpness $\mathcal{S}_R$ as a function of batch size and learning rate for MNIST (top) and CIFAR-10 (bottom), confirming the inverse $\eta/B$ scaling ($\mathcal{S}_R \propto B/\eta$) predicted by Theorem~\ref{thm:main}.}
    \label{fig:bs_ablation}
\end{figure*}

\begin{figure*}[t]
    \centering
    \includegraphics[width=\linewidth]{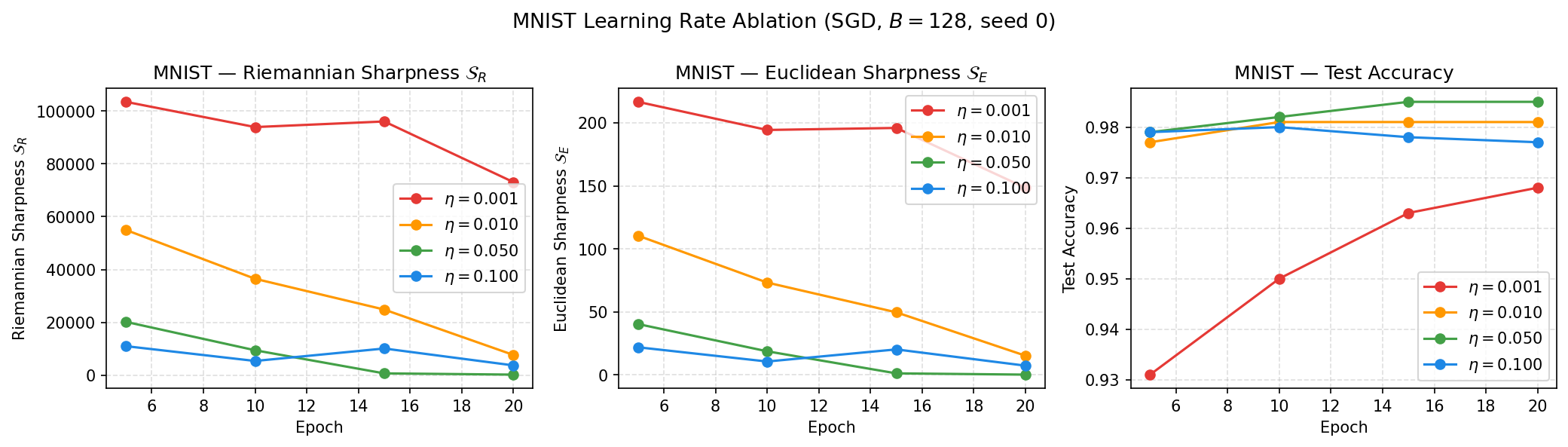}
    \includegraphics[width=\linewidth]{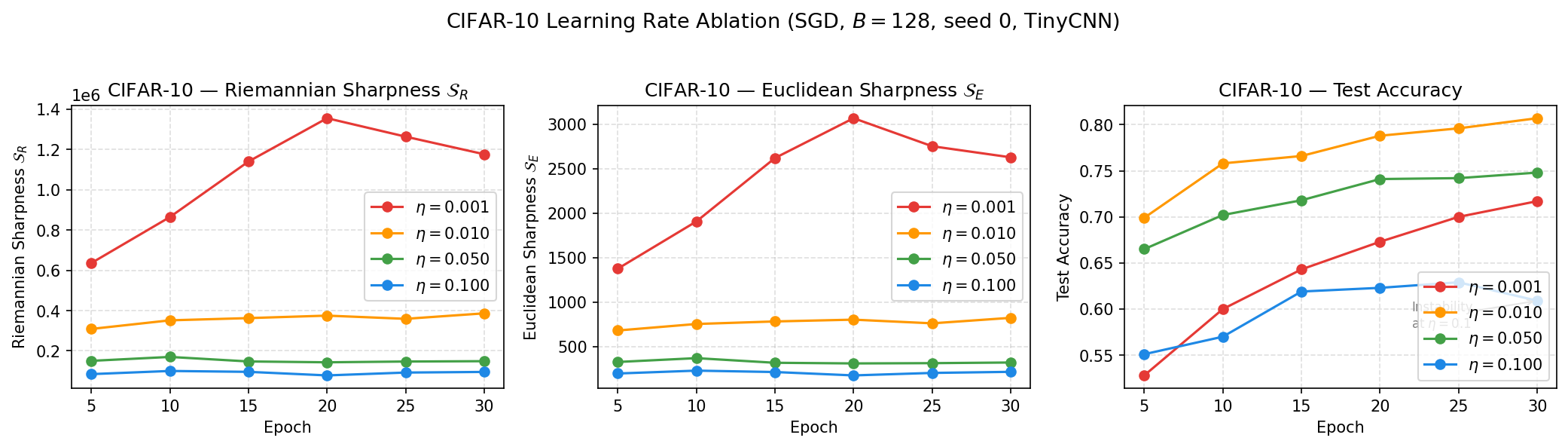}
    \caption{Riemannian sharpness $\mathcal{S}_R$ as a function of
    learning rate at fixed batch size $B = 128$.}
    \label{fig:lr_ablation}
\end{figure*}

\subsection{Learning Rate Scaling}

Fixing $B = 128$ and varying $\eta$ reveals a clear relationship between $\eta$ and $\mathcal{S}_R$ (Table~\ref{tab:lr_ablation}). On MNIST, larger learning rates produce smaller $\mathcal{S}_R$ and better generalization, consistent with Theorem~\ref{thm:main}: at $\eta=0.001$, $\mathcal{S}_R \approx 73{,}102$ with test accuracy $0.968$, while at $\eta=0.05$, $\mathcal{S}_R$ drops to $168$ and accuracy rises to $0.985$. The one exception is $\eta=0.1$, which gives a higher $\mathcal{S}_R$ ($3{,}668$) than $\eta=0.05$ and slightly worse accuracy ($0.977$), suggesting that at very large learning rates SGD enters a regime where the SDE approximation breaks down and the noise is no longer simply regularizing.

On CIFAR-10 the pattern reverses at large $\eta$: accuracy peaks at $\eta=0.01$ ($0.807$) and falls sharply at $\eta=0.1$ ($0.609$), while $\mathcal{S}_R$ continues to decrease. This divergence — lower $\mathcal{S}_R$ but worse accuracy — occurs because $\eta=0.1$ is too large for stable optimization on this task with TinyCNN; the model fails to converge properly rather than finding a genuinely flat and well-generalizing minimum. This reinforces an important caveat: $\mathcal{S}_R$ is a reliable predictor of generalization when the optimizer has converged, but very large learning rates can produce low $\mathcal{S}_R$ at unstable or suboptimal solutions.

\begin{table}[h]
\centering
\caption{Learning rate ablation (SGD, $B=128$, final epoch, seed~0). On MNIST, 
larger $\eta$ consistently lowers $\mathcal{S}_R$ and improves generalization. 
On CIFAR-10, $\eta=0.1$ achieves low $\mathcal{S}_R$ but poor accuracy due to 
training instability.}
\label{tab:lr_ablation}
\begin{tabular}{clcccc}
\hline
Dataset & $\eta$ & $\eta/B$ & $\mathcal{S}_R$ & Test Acc \\
\hline
\multirow{4}{*}{MNIST}
 & $0.001$ & $7.81\times10^{-6}$ & $73{,}102.1$ & $0.968$ \\
 & $0.010$ & $7.81\times10^{-5}$ & $7{,}629.1$  & $0.981$ \\
 & $0.050$ & $3.91\times10^{-4}$ & $167.9$      & $0.985$ \\
 & $0.100$ & $7.81\times10^{-4}$ & $3{,}668.1$  & $0.977$ \\
\hline
\multirow{4}{*}{CIFAR-10}
 & $0.001$ & $7.81\times10^{-6}$ & $1{,}176{,}416.0$ & $0.717$ \\
 & $0.010$ & $7.81\times10^{-5}$ & $386{,}314.4$     & $0.807$ \\
 & $0.050$ & $3.91\times10^{-4}$ & $148{,}639.5$     & $0.748$ \\
 & $0.100$ & $7.81\times10^{-4}$ & $95{,}286.9$      & $0.609$ \\
\hline
\end{tabular}
\end{table}

\subsection{Training Dynamics}

Beyond final-epoch values, the training trajectories in Table~\ref{tab:dynamics_mnist} reveal how $\mathcal{S}_R$ evolves during optimization. On MNIST, SGD small-batch ($B=32$) shows monotonically decreasing sharpness: from $15{,}232$ at epoch~5 (seed~0) down to $190$ at epoch~20. This is consistent with the SDE stationary distribution interpretation — the optimizer progressively escapes sharp basins and settles into flatter regions as training continues. Large-batch SGD ($B=512$) follows a different trajectory. Sharpness decreases more slowly and stabilizes at a high plateau ($63{,}909$--$87{,}570$ depending on seed), while test accuracy improves more gradually. The reduced noise temperature $\eta/B$ is insufficient to drive the optimizer into genuinely flat basins.

Adam on MNIST also converges to low sharpness by epoch~20 ($2{,}575$--$6{,}095$ across seeds) and achieves high accuracy ($0.980$ averaged over seeds), slightly above SGD small-batch. The steep early descent in sharpness suggests Adam's adaptive preconditioning accelerates the implicit bias toward flat minima.

On CIFAR-10 (Table~\ref{tab:dynamics_cifar}), SGD with $B=128$ shows sharpness increasing slightly over training ($308{,}825$ at epoch~5 to $386{,}314$ at epoch~30, seed~0) while accuracy improves steadily ($0.699 \to 0.807$). Adam on CIFAR-10 exhibits more variable dynamics, with sharpness stabilizing around $232{,}765$ by epoch~30 (seed~0; $226{,}738$ averaged over seeds) and accuracy reaching $0.794$ on average — a profile consistent with a different geometry of descent than FIM-based SGD.

\begin{table}[h]
\centering
\caption{Training dynamics on MNIST. $\mathcal{S}_R$ and $\mathcal{S}_E$ reported at epochs 5, 10, 15, 20 for small-batch SGD, large-batch SGD, and Adam (all seed~0).}
\label{tab:dynamics_mnist}
\begin{tabular}{llccccc}
\hline
Optimizer & Epoch & Test Acc & $\mathcal{S}_R$ & $\mathcal{S}_E$ \\
\hline
\multirow{4}{*}{SGD ($B=32$)}
 & 5  & $0.979$ & $15{,}232.3$  & $30.8$ \\
 & 10 & $0.984$ & $4{,}306.6$   & $8.6$ \\
 & 15 & $0.985$ & $620.2$       & $1.2$ \\
 & 20 & $0.985$ & $190.0$       & $0.4$ \\
\hline
\multirow{4}{*}{SGD ($B=512$)}
 & 5  & $0.955$ & $87{,}570.1$  & $176.3$ \\
 & 10 & $0.969$ & $82{,}920.9$  & $166.5$ \\
 & 15 & $0.975$ & $80{,}536.7$  & $161.5$ \\
 & 20 & $0.977$ & $63{,}909.2$  & $128.1$ \\
\hline
\multirow{4}{*}{Adam ($B=128$)}
 & 5  & $0.980$ & $15{,}246.6$  & $30.9$ \\
 & 10 & $0.981$ & $7{,}882.7$   & $16.0$ \\
 & 15 & $0.980$ & $14{,}536.3$  & $29.5$ \\
 & 20 & $0.981$ & $3{,}015.1$   & $6.0$ \\
\hline
\end{tabular}
\end{table}

\begin{table}[h]
\centering
\caption{Training dynamics for TinyCNN on CIFAR-10 (SGD $B=128$ and Adam $B=128$, seed~0). Sharpness values are FIM-based ($\mathcal{S}_R$) and Euclidean ($\mathcal{S}_E$).}
\label{tab:dynamics_cifar}
\begin{tabular}{llccccc}
\hline
Optimizer & Epoch & Test Acc & $\mathcal{S}_R$ & $\mathcal{S}_E$ \\
\hline
\multirow{6}{*}{SGD ($B=128$)}
 & 5  & $0.699$ & $308{,}824.8$ & $679.6$ \\
 & 10 & $0.758$ & $351{,}542.6$ & $752.7$ \\
 & 15 & $0.766$ & $362{,}620.0$ & $781.3$ \\
 & 20 & $0.788$ & $374{,}911.9$ & $801.4$ \\
 & 25 & $0.796$ & $358{,}997.5$ & $760.0$ \\
 & 30 & $0.807$ & $386{,}314.4$ & $821.9$ \\
\hline
\multirow{6}{*}{Adam ($B=128$)}
 & 5  & $0.711$ & $237{,}225.5$ & $965.2$ \\
 & 10 & $0.753$ & $244{,}045.6$ & $1135.9$ \\
 & 15 & $0.763$ & $234{,}513.6$ & $1157.2$ \\
 & 20 & $0.774$ & $237{,}633.5$ & $1241.8$ \\
 & 25 & $0.778$ & $237{,}011.7$ & $1078.3$ \\
 & 30 & $0.789$ & $232{,}765.0$ & $1072.1$ \\
\hline
\end{tabular}
\end{table}

\begin{table}[h]
\centering
\caption{Summary of final metrics across all configurations (epoch 20 for MNIST, epoch 30 for CIFAR-10; mean over 3 seeds where available).
$\mathcal{S}_R$ is FIM-based Riemannian sharpness;
$\mathcal{S}_E$ is Euclidean sharpness.}
\label{tab:summary}
\begin{tabular}{lccc}
\hline
Configuration & Test Acc & $\mathcal{S}_R$ & $\mathcal{S}_E$ \\
\hline
SGD small-batch (MNIST, $B=32$)    & $0.985$ & $388.0$           & $0.8$ \\
SGD large-batch (MNIST, $B=512$)   & $0.977$ & $71{,}656.1$      & $143.7$ \\
Adam (MNIST, $B=128$)              & $0.980$ & $3{,}895.4$       & $7.8$ \\
TinyCNN SGD (CIFAR-10, $B=128$)    & $0.803$ & $386{,}247.4$     & $825.7$ \\
TinyCNN Adam (CIFAR-10, $B=128$)   & $0.794$ & $226{,}737.6$     & $1{,}157.7$ \\
\hline
\end{tabular}
\end{table}

\begin{figure*}[t]
    \centering
    \includegraphics[width=\linewidth]{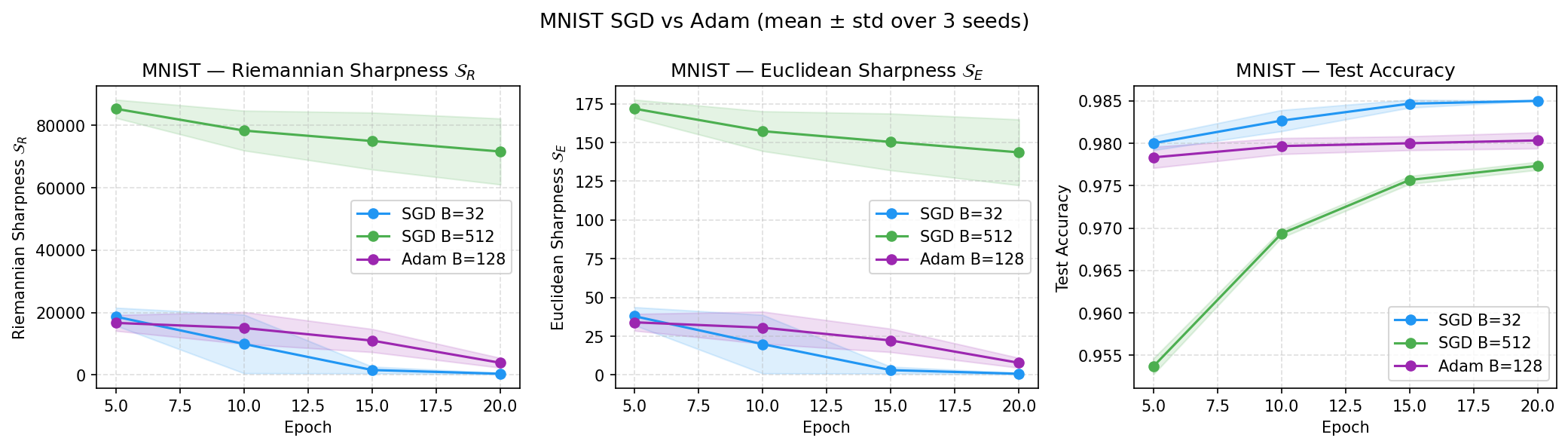}
    \caption{Comparison of Fisher Information Matrix-based
    Riemannian sharpness $\mathcal{S}_R$ and Euclidean
    Hessian sharpness}
    \label{fig:fim_vs_hs}
\end{figure*}

\subsection{SGD vs.\ Adam}
\label{sec:sgd_vs_adam}

Table~\ref{tab:summary} compares SGD and Adam across MNIST and CIFAR-10. On MNIST, Adam achieves low $\mathcal{S}_R$ ($3{,}895$ on average over seeds) and high accuracy ($0.980$), comparable to SGD small-batch. This is consistent with Adam's adaptive preconditioning implicitly optimizing for flatness in the FIM-weighted geometry, even though the SDE framework of Theorem~\ref{thm:main} was derived for SGD with Fisher-structured noise.

On CIFAR-10, Adam achieves slightly lower $\mathcal{S}_R$ ($226{,}738$) than SGD ($386{,}247$) but also slightly lower accuracy ($0.794$ vs.\ $0.803$). The two metrics are not straightforwardly comparable across optimizers, because Adam's preconditioning changes the effective metric under which sharpness is being measured. To investigate, we compute the Adam-adjusted Riemannian sharpness,
\begin{equation}
\mathcal{S}_R^{\mathrm{Adam}}
  = \mathrm{tr}\!\left(G_{\mathrm{Adam}}^{-1} H\right),
\quad
G_{\mathrm{Adam}} = \mathrm{diag}(v_t),
\label{eq:adam_sharp}
\end{equation}
where $v_t$ is Adam's second-moment diagonal. This captures the flatness of the minimum with respect to Adam's own preconditioning geometry rather than the FIM.
\begin{remark}
    This supports a broader interpretation: the implicit bias of any preconditioned optimizer is toward minima flat with respect to the optimizer's own metric tensor, of which the SGD/FIM result is a special case.
\end{remark}

\subsection{Reparametrization Invariance}
\label{sec:reparam}

Lemma~\ref{lem:reparam} guarantees that $\mathcal{S}_R$ is invariant under smooth reparametrizations when the true FIM is used. However in practice the empirical diagonal FIM estimator is not fully reparametrization-invariant. Under a layer-wise weight rescaling that preserves the network's input--output function exactly, $\mathcal{S}_E$ increases dramatically while $\mathcal{S}_R$ remains substantially more stable --- yet still deviates from exact invariance.

This gap has a clear mechanical explanation: rescaling the weights changes gradient magnitudes, which in turn changes the independently re-estimated empirical FIM at the rescaled parameters. The empirical Fisher therefore does \emph{not} transform as a $(0,2)$ tensor under reparametrization, and the cancellation exploited in the proof of Lemma~\ref{lem:reparam} fails. This is not a failure of the theory --- the invariance guarantee of Lemma~\ref{lem:reparam} applies to the \emph{true} FIM, which is a property of the predictive distribution alone and is therefore parameterization-independent.

The limitations of the empirical Fisher as a proxy for the true FIM are well-documented \cite{kunstner2019limitations}, and exact invariance under arbitrary reparametrizations would require a structured estimator such as K-FAC \cite{martens2015optimizing} that captures the full covariance rather than only the diagonal. The natural-gradient literature \cite{pascanu2014revisiting} similarly distinguishes between theoretical properties of the true FIM and the behavior of tractable approximations. We regard the development of reparametrization-robust empirical FIM estimators as an important avenue for future work, and we caution the reader that the invariance claim in our abstract and contributions is a property of $\mathcal{S}_R$ defined with the \emph{true} FIM; the empirical $\hat{\mathcal{S}}_R$ used in all experiments inherits this invariance only approximately.

\section{Conclusion}

We have presented a geometric framework for understanding the implicit bias of mini-batch SGD through the lens of information geometry. By replacing the Euclidean metric on parameter space with the Fisher Information Matrix, we have defined a reparametrization-invariant sharpness measure — Riemannian sharpness $\mathcal{S}_R$ — which addresses the fundamental objection to the flatness-generalization narrative raised by \cite{dinh2017sharp}. Our main results establish that (i) $\mathcal{S}_R$ is invariant under smooth function-preserving reparametrizations (Lemma~1); (ii) the stationary distribution of mini-batch SGD assigns exponentially greater probability mass to Riemannian-flat minima (Theorem~1, Corollary~1); and (iii) $\mathcal{S}_R$ controls generalization through a PAC-Bayes bound (Corollary~2).

Empirically, $\mathcal{S}_R$ is a reliable predictor of generalization across batch sizes, learning rates, and optimizers on both MNIST and CIFAR-10. The scaling of $\mathcal{S}_R$ with $\eta/B$ confirms Theorem~1 on MNIST, where the scalar claim of Assumption~\ref{ass:noise} holds closely. On CIFAR-10 the directional alignment between the gradient noise covariance and the FIM is very high (cosine $> 0.984$), but the $\eta/B$ scaling is less clean, particularly at large learning rates where training instability confounds the geometric interpretation. The behavior of Adam suggests a natural extension of our framework to preconditioned optimizers, where the relevant flatness measure is defined with respect to the optimizer's own preconditioning matrix rather than the FIM.

\textbf{Limitations.} Assumption~\ref{ass:noise} (Fisher noise structure) is an approximation that holds well directionally but whose scalar magnitude varies with architecture and dataset. The diagonal FIM approximation discards off-diagonal correlations, and the continuous-time SDE limit requires sufficiently small $\eta$. The empirical $\hat{\mathcal{S}}_R$ is not fully reparametrization-invariant in practice, though it remains substantially more stable than $\mathcal{S}_E$ under rescaling, as discussed in Section~\ref{sec:reparam}. On CIFAR-10, the relationship between $\mathcal{S}_R$ and generalization is not monotone across all learning rates, indicating that the metric is most informative when the optimizer has converged to a genuine minimum. Future work should extend validation to larger architectures and structured (Kronecker-factored) FIM estimates.

\textbf{Broader Impact.} Our results provide a theoretically grounded justification for small-batch training as a geometric regularization strategy and suggest that flatness-aware objectives such as SAM \cite{foret2021sharpness} may be most naturally formulated in the Riemannian geometry of parameter space.


\section{Acknowledgment}
The authors declare no external funding for this work. During the preparation of this manuscript, the authors used AI-assisted tools (Claude, DeepSeek) for coding and debugging assistance, and Grammarly for grammar and language suggestions. All theoretical development, experimental design, analysis, and conclusions are the sole responsibility of the authors.

\bibliographystyle{plainnat}

\bibliography{references}

@article{hochreiter1997flat,
  author    = {S. Hochreiter and J. Schmidhuber},
  title     = {Flat minima},
  journal   = {Neural Computation},
  volume    = {9},
  number    = {1},
  pages     = {1--42},
  year      = {1997}
}

@inproceedings{chaudhari2017entropy,
  author    = {P. Chaudhari and A. Choromanska and S. Soatto and Y. LeCun and C. Baldassi and C. Borgs and J. Chayes and L. Sagun and R. Zecchina},
  title     = {Entropy-{SGD}: Biasing Gradient Descent into Wide Valleys},
  booktitle = {Proceedings of the International Conference on Learning Representations (ICLR)},
  year      = {2017}
}

@article{amari1998natural,
  author    = {S. Amari},
  title     = {Natural Gradient Works Efficiently in Learning},
  journal   = {Neural Computation},
  volume    = {10},
  number    = {2},
  pages     = {251--276},
  year      = {1998}
}

@inproceedings{mulayoff2020implicit,
  author    = {R. Mulayoff and T. Michaeli},
  title     = {Unique Properties of Flat Minima in Deep Networks},
  booktitle = {Proceedings of the International Conference on Machine Learning (ICML)},
  year      = {2020}
}

@inproceedings{nacson2022implicit,
  author    = {M. S. Nacson and K. Ravikumar and N. Srebro and D. Soudry},
  title     = {Implicit Bias of {SGD} for Diagonal Linear Networks: A Provable Benefit of Stochasticity},
  booktitle = {Advances in Neural Information Processing Systems (NeurIPS)},
  year      = {2022}
}

@inproceedings{dinh2017sharp,
  author    = {L. Dinh and R. Pascanu and S. Bengio and Y. Bengio},
  title     = {Sharp Minima Can Generalize for Deep Nets},
  booktitle = {Proceedings of the International Conference on Machine Learning (ICML)},
  year      = {2017}
}

@inproceedings{li2017stochastic,
  author    = {Q. Li and C. Tai and W. E},
  title     = {Stochastic Modified Equations and Adaptive Stochastic Gradient Algorithms},
  booktitle = {Proceedings of the International Conference on Machine Learning (ICML)},
  year      = {2017}
}

@inproceedings{mcallester1999pac,
  author    = {D. A. McAllester},
  title     = {{PAC-Bayesian} Model Averaging},
  booktitle = {Proceedings of the Conference on Computational Learning Theory (COLT)},
  year      = {1999}
}

@inproceedings{kunstner2019limitations,
  title     = {Limitations of the Empirical {Fisher} Approximation for Natural Gradient Descent},
  author    = {Kunstner, Frederik and Balles, Lukas and Hennig, Philipp},
  booktitle = {Advances in Neural Information Processing Systems},
  volume    = {32},
  pages     = {4156--4167},
  year      = {2019},
  publisher = {Curran Associates, Inc.}
}

@inproceedings{pascanu2014revisiting,
  title     = {Revisiting Natural Gradient for Deep Networks},
  author    = {Pascanu, Razvan and Bengio, Yoshua},
  booktitle = {International Conference on Learning Representations},
  year      = {2014}
}

@inproceedings{dziugaite2017computing,
  author    = {G. K. Dziugaite and D. M. Roy},
  title     = {Computing Nonvacuous Generalization Bounds for Deep (Stochastic) Neural Networks with Many Parameters},
  booktitle = {Proceedings of the Conference on Uncertainty in Artificial Intelligence (UAI)},
  year      = {2017}
}

@book{cencov1982statistical,
  author    = {N. N. {\v{C}}encov},
  title     = {Statistical Decision Rules and Optimal Inference},
  publisher = {American Mathematical Society},
  address   = {Providence, RI},
  year      = {1982}
}

@inproceedings{keskar2017large,
  author    = {N. S. Keskar and D. Mudigere and J. Nocedal and M. Smelyanskiy and P. T. P. Tang},
  title     = {On Large-Batch Training for Deep Learning: Generalization Gap and Sharp Minima},
  booktitle = {Proceedings of the International Conference on Learning Representations (ICLR)},
  year      = {2017}
}

@inproceedings{neyshabur2017exploring,
  author    = {B. Neyshabur and S. Bhojanapalli and D. McAllester and N. Srebro},
  title     = {Exploring Generalization in Deep Learning},
  booktitle = {Advances in Neural Information Processing Systems (NeurIPS)},
  pages     = {5949--5958},
  year      = {2017}
}

@inproceedings{foret2021sharpness,
  author    = {P. Foret and A. Kleiner and H. Mobahi and B. Neyshabur},
  title     = {Sharpness-Aware Minimization for Efficiently Improving Generalization},
  booktitle = {Proceedings of the International Conference on Learning Representations (ICLR)},
  year      = {2021}
}

@inproceedings{martens2015optimizing,
  author    = {J. Martens and R. Grosse},
  title     = {Optimizing Neural Networks with {Kronecker}-Factored Approximate Curvature},
  booktitle = {Proceedings of the International Conference on Machine Learning (ICML)},
  year      = {2015}
}

@inproceedings{wilson2020marginal,
  author    = {A. G. Wilson and P. Izmailov},
  title     = {Bayesian Deep Learning and a Probabilistic Perspective of Generalization},
  booktitle = {Advances in Neural Information Processing Systems (NeurIPS)},
  year      = {2020}
}

@inproceedings{jang2022reparametrization,
  author    = {C. Jang and S. Lee and F. C. Park and Y.-K. Noh},
  title     = {A Reparametrization-Invariant Sharpness Measure Based on Information Geometry},
  booktitle = {Advances in Neural Information Processing Systems (NeurIPS)},
  year      = {2022}
}

@inproceedings{kristiadi2023geometry,
  author    = {A. Kristiadi and F. Dangel and P. Hennig},
  title     = {The Geometry of Neural Nets' Parameter Spaces under Reparametrization},
  booktitle = {Advances in Neural Information Processing Systems (NeurIPS)},
  year      = {2023}
}

\vskip 0.2in

\end{document}